\documentclass[10pt,journal,compsoc]{IEEEtran}
\usepackage[utf8]{inputenc}
\ifCLASSOPTIONcompsoc
  \usepackage[nocompress]{cite}
\else
  \usepackage{cite}
\fi

\ifCLASSINFOpdf
   \usepackage[pdftex]{graphicx}
\else
   \usepackage[dvips]{graphicx}
\fi

\usepackage{amsmath}

\usepackage{algorithmic}
\usepackage{array}

\ifCLASSOPTIONcompsoc
 \usepackage[caption=false,font=footnotesize,labelfont=sf,textfont=sf]{subfig}
\else
 \usepackage[caption=false,font=footnotesize]{subfig}
\fi

\usepackage{stfloats}

\ifCLASSOPTIONcaptionsoff
 \usepackage[nomarkers]{endfloat}
\let\MYoriglatexcaption\caption
\renewcommand{\caption}[2][\relax]{\MYoriglatexcaption[#2]{#2}}
\fi

\usepackage{url}

\usepackage{color}
\usepackage{array}
\usepackage{lipsum}
\usepackage{float}
\usepackage{booktabs}
\usepackage{multirow}
\usepackage{amsmath}
\usepackage{amsthm}
\usepackage{amssymb}
\usepackage{graphicx}
\usepackage{mathtools}
\usepackage[breaklinks=true,bookmarks=false]{hyperref}
\usepackage{hyperref}
\hypersetup{
    colorlinks=true,
    linkcolor=black,     
    urlcolor=black,
}
\usepackage{bbm}
\usepackage{algorithm}
\usepackage{algorithmic}
\usepackage{multirow}
\usepackage{xcolor}
\usepackage{pifont}

\newtheorem{theorem}{Theorem}[section]

\newtheorem{remark}[theorem]{Remark}

\newcommand{\cmark}{\ding{51}}%
\newcommand{\xmark}{\ding{55}}%
\newlength\savewidth\newcommand\shline{\noalign{\global\savewidth\arrayrulewidth
  \global\arrayrulewidth 1pt}\hline\noalign{\global\arrayrulewidth\savewidth}}

\title{PatchMix Augmentation to Identify Causal Features in Few-shot Learning }
\author{Chengming Xu$^{*}$, Chen Liu$^{*}$, Xinwei Sun, Siqian Yang, Yabiao Wang, Chengjie Wang, Yanwei Fu

\IEEEcompsocitemizethanks{
\IEEEcompsocthanksitem $^{*}$ indicates equal contributions. Xinwei Sun and Yanwei Fu are the corresponding authors.  This work was supported in part by the National Natural Science Foundation of China Grant (62076067), Shanghai Municipal Science and Technology Major  Project (2021SHZDZX0103), and ZJ Lab.
\IEEEcompsocthanksitem Chengming Xu, Xinwei Sun and Yanwei Fu are with the School of Data Science and MOE Frontiers Center for Brain Science, Fudan University. Yanwei Fu is also with Fudan ISTBI—ZJNU Algorithm Centre for Brain-inspired Intelligence, Zhejiang Normal University, Jinhua, China.  E-mail: \{cmxu18, sunxinwei, yanweifu\}@fudan.edu.cn
\IEEEcompsocthanksitem Chen Liu is with the Department of Mathematics at the Hong Kong University of Science and Technology. E-mail: cliudh@connect.ust.hk
\IEEEcompsocthanksitem Siqian Yang, Yabiao Wang and Chengjie Wang are with Youtu Lab, Tencent. Chengjie Wang is also with Shanghai Jiao Tong Universtity. E-mail: \{seasonsyang, caseywang, jasoncjwang\}@tencent.com
}
}

\begin{document}

\IEEEtitleabstractindextext{
\begin{abstract}

The task of Few-shot learning (FSL) aims to transfer the knowledge learned from base categories with {sufficient} 
labelled data to novel categories {with scarce known information}. It is currently an important research question and has great practical values in the real-world applications. Despite extensive previous efforts  are made on few-shot learning tasks, we emphasize that most existing methods did not take into account the distributional shift caused by sample selection bias in the FSL scenario.
Such a selection bias can induce spurious correlation between the semantic \emph{causal} features, that are causally and semantically related to the class label, and the other  \emph{non-causal} features. Critically, the former ones should be invariant across changes in distributions, highly related to the classes of interest, and thus well generalizable to novel classes,  while 
the latter ones are not stable to changes in the distribution. 
To resolve this problem, we propose a novel data augmentation strategy dubbed as \emph{PatchMix} that can break this spurious dependency by replacing the patch-level information and supervision of the query images with random gallery images from different classes from the query ones. We theoretically show that such an augmentation mechanism, different from existing ones, is able to identify the causal features. To further make these features to be discriminative enough for classification, we propose \emph{Correlation-guided Reconstruction} (CGR) and \emph{Hardness-Aware} module for instance discrimination and easier discrimination between similar classes. Moreover, such a framework can be adapted to the unsupervised FSL scenario. The utility of our method is demonstrated on the state-of-the-art results consistently achieved on several benchmarks including \textit{mini}ImageNet, \textit{tiered}ImageNet, CIFAR-FS , CUB, Cars, Places and Plantae, in all settings of single-domain, cross-domain and unsupervised FSL. By studying the intra-variance property of learned features and visualizing the learned features, we further quantitatively and qualitatively show that such a promising result is due to the effectiveness in learning causal features.

\end{abstract}

}

\maketitle

\section{Introduction \label{sec:intro}}
Among many factors to the successful deep learning applications of modern society, it is necessary and decisive to collect a large amount of labeled training data. Typically, prevailing computer vision models such as ResNet~\cite{he2016deep} and Faster R-CNN~\cite{ren2015faster} are  trained by millions of labeled examples and thus achieve decent generalization ability.
Unfortunately, for some cases such as the rare species, it is not feasible to collect large amount of labeled and diverse data which can be used for training.

Motivated by humans' ability of \textit{learning to learn} new objects/concepts with few references, the task of Few-Shot Learning (FSL) is recently studied in the computer vision and machine learning communities. Generally, the FSL aims at learning a model, which can generalize to  \textit{novel/target} dataset with  few labelled data (\emph{i.e.}, support sample) available, on \textit{base/source} dataset of vast annotated data. 

The existing FSL methods \cite{snell2017prototypical, sung2018learning, finn2017model} exploit the knowledge transferring from the base to novel categories via meta-learning paradigm. 
One of the most commonly used supervision signals by these meta-learning methods is the image class labels from base dataset~\cite{hou2019cross}. 
This is the same as many-shot learning. And particularly in classical many-shot learning, the guidance from class labels is common and effective under the independent and identically distributed (\emph{i.i.d}) assumption.
Unfortunately, the FSL is suffered from the problem of \textit{distribution shift}~(Chap.20 in~\cite{probablistic2022kevin}); and the testing distribution of novel target classes is quite different from the training distribution of those source/base classes.
This is caused by the sample selection bias existed in data collection and support/query set splitting. 

\begin{figure}
    \centering
    \includegraphics[scale=0.8]{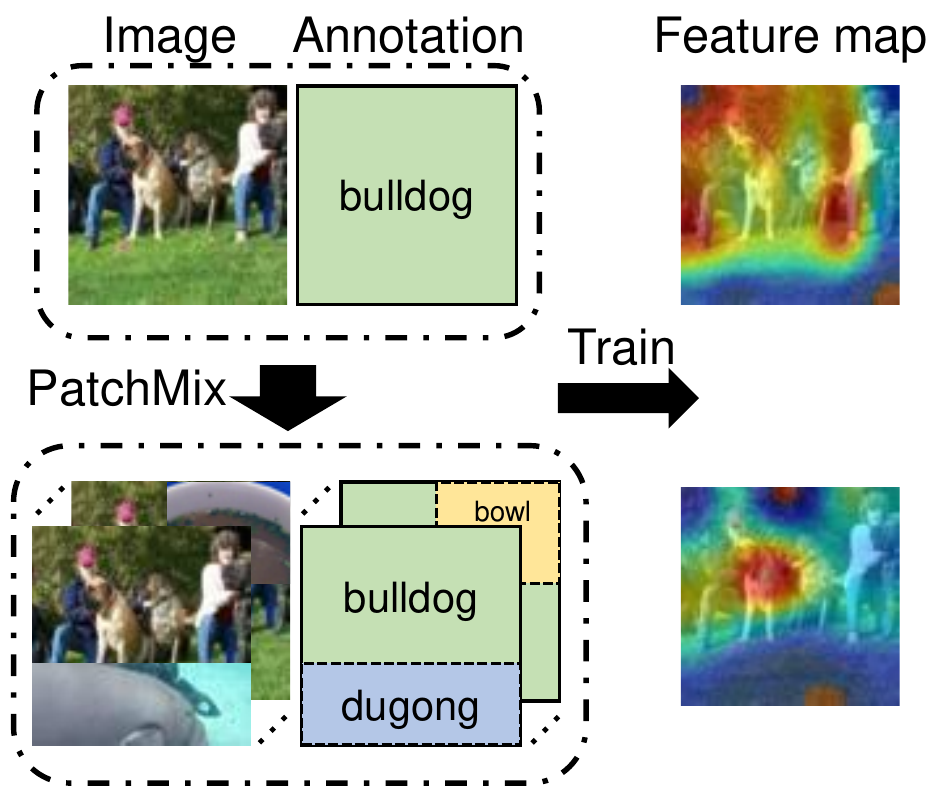}
    \caption{The existing FSL methods typically train the models directly with image-level annotations. In this way both the causal (\emph{e.g.} the dogs) and non-causal features (\emph{e.g.} the people, trees and grass) are learned to build up correlation with the class label, which can lose prediction ability and be less effective on novel data due to the domain gap. Compared with those methods, our proposed PatchMix can help the model learn disentangled causal features by breaking up the dependency between causal and non-causal features, thus making better generalization to novel categories.    \label{fig:teaser}}
 
\end{figure}

Such a selection bias can induce spurious correlation between causal and non-causal features of the classes of interest. These two kinds of features causally and not causally influence the class label respectively. For example, if the label is "dog", its texture, shape are causal features while those features that are related to the "people" are non-causal features. 
Critically, the causal features should be reliably predictive of the classes of interest,  invariant across changes in distributions, and thus  well generalizable to novel classes. As deep models  are typically optimized to well fit the training data, such "shortcut non-causal features'' would be easier to be learned by inheriting the   spurious correlation~\cite{xiao2021iclr,shah2020nips}. While the non-causal features may be still good to the supervised learning on the i.i.d data~\cite{khani2021removing}, it may hurt the performance of generalizing to novel data, which has different distribution from the base data.

We give a detailed example in Fig.~\ref{fig:teaser}. As shown in the feature map, the model has learned to highly associate the grass and persons with the dogs to help recognize the `dog' class. This makes sense, as  obviously  people would like to play with dog on the grass quite often. This leads to the biased combination of these elements in the sampled training data. Nevertheless, the learned non-causal features, \emph{i.e.}, the grass and the people, are no longer correlated with other categories like horse, lion anymore. Compared with IFSL~\cite{yue2020interventional} which also considers spurious correlation but only partly studies the bias inherited from the pre-trained backbones, our paper will understand and address this challenge in the general distribution shift in FSL. Even when there is no pre-training stage, the spurious correlation can still commonly exists via such a shift. 

To this end, we propose enforcing a novel data-augmentation mechanism dubbed as PatchMix that are effective in learning causal features and hence better generalization to novel categories. Specifically, our PatchMix replaces some patches from each query image with class label $A$ with a set of gallery images from a random different class $B$. Meanwhile, the replaced patches are labelled as $B$ (\emph{i.e.}, gallery category), rather than the query category, \emph{i.e.}, $A$. In this regard, the spurious correlation between causal and non-causal factors that often come up from different patches (\emph{e.g.}, the dog and the people) is largely weaken, endowing the model with the ability to disentangle the causal features from others. In particular, we provide a theoretical analysis from \emph{Structural Causal Model} (SCM). We show that our PatchMix operation has the ability of  reducing correlation between spurious/non-causal and causal features; thus it is also able to defunction the non-causal features during testing on novel classes. Such a disentanglement can make our PatchMix prominent among other data augmentation methods like CutMix~\cite{yun2019cutmix} (see Sec.~\ref{sec:anaylysis} for detailed analysis). After dropping the non-causal features, our learned features can correlate better with class labels, and thus have smaller variances among instances from the same novel category. As the result, models trained with PatchMix can enjoy an 
easy classification between different classes. 

To make our causal features more discriminative for classification, we further propose two modules to enhance the PatchMix, namely \emph{hardness-aware} and \emph{Correlation-Guided Reconstruction} (CGR) modules. Specifically, in CGR, original query image are reconstructed by selecting informative patches from both query and gallery image based on similarity between each patch and the query patches. This module can help our model to tell apart features from different images, thus fulfilling instance discrimination equipped with better learned features. To further learn discriminative features especially between similar classes, we in the \emph{hardness-aware} module propose selecting the query and gallery images from two classes that are globally most similar to each other. Specifically, we formulate this problem into a Travelling Salesman Problem (TSP) on the distance graph, in which
categories and negative similarity among them are formulated as nodes and edges of the graph. Equipped with such a mixture strategy, we can control the hardness of training episodes, thus leading to more robustness and better representations.

Moreover, based on CACTUs\cite{hsu2018unsupervised}, our PatchMix can be well adapted to unsupervised learning FSL (\emph{i.e.}, only unlabeled base data is provided), which is meaningful when the labeling cost is high. Concretely, the unsupervised pre-training stage is replaced by our novel PatchMoCo which constrains the patch-level contrast among different images. Then we involve our PatchMix into the pseudo-label training stage.

In order to validate the efficacy of our method, we conduct experiments on single-domain, cross-domain and unsupervised FSL with a wide range of benchmark datasets including \textit{mini}ImageNet, \textit{tiered}ImageNet, CIFAR-FS and CUB. Extensive results show that PatchMix can achieve the state-of-the-art performance on all settings comparing with previous methods, implying the generalization ability. Besides, we show that such an improvement can be contributed to the ability of learning causal features, indicated by the concentration of class of interest in visualized feature map and smaller intra-variance of learned features among instances from the same category (\emph{a.k.a}, neural collapse in \cite{galanti2021role}).

The contributions of this paper can be listed as follows: (1) We propose a novel data augmentation method tailored for few-shot learning dubbed PatchMix, that can remove the spurious correlation, and identify the causal features. Critically,  we leverage the well-established causal framework--\emph{Structural Causal Model} to help explain and understand our PatchMix. To the best of our knowledge, it is the first work that presents a unified causal learning framework for data augmentation to identify causal features in the few-shot learning tasks. (2) We enhance our PatchMix model with the newly proposed \emph{correlation-guided reconstruction} and \emph{hardness-aware} modules. These new modules can facilitate  better feature learning. (3) We propose a novel unsupervised FSL method including a new unsupervised pre-training strategy named PatchMoCo, which is built upon the PatchMix, and a pseudo-label training stage.
    (4) We conduct a vast amount of experiments on different settings and various datasets. The results verify the effective of our proposed method.

\noindent\textbf{Extensions.} This paper is an extension of \cite{liu2021learning}. We have added the following aspects based on the conference version: (1) We provide a way to explain how PatchMix can help few-shot learning in the perspective of learning disentangled causal features by providing a theoretical interpretation. 
    (2) We further enhance our PatchMix with two novel modules, \emph{i.e}. correlation-guided reconstruction and hardness-aware PatchMix. These two modules further improve the efficacy of our PatchMix model.
    (3) We empirically show that the proposed CGR and hardness-aware PatchMix can significantly improve the model. Meanwhile, our proposed method can receive state-of-the-art results on several settings. The codes\&models will be released on the project pages.

\section{Related Work}
\noindent\textbf{Few-shot recognition.} Few-shot learning (FSL)  
aims to recognize the target classes by adapting the prior ‘knowledge’ learned from {base categories}. Such knowledge usually resides in a deep embedding model for the general-purpose matching of the support and query image pairs. The embedding is normally learned with enough training instances on {base categories} and updated by a few training instances on {novel categories}. Recent efforts for FSL are made on optimization, metric-learning  and augmentation.

Optimization based methods ~\cite{ravi2016optimization, finn2017model, nichol2018first, li2017meta, sun2019meta, rusu2018meta, li2019learning, peng2019few, xing2019adaptive} learn on the base dataset a good initialization that can be quickly adapted to novel dataset. Metric-learning based methods~\cite{snell2017prototypical, hou2019cross, ye2020feat, zhang2020deepemd, sung2018learning, hu2020leveraging, snell2020bayesian, fei2020melr, zhang2020iept, wertheimer2021few, oh2020boil, su2020does, wu2020attentive, dhillon2019baseline, afrasiyabi2020associative, li2020adversarial, zhang2019variational} learn a good embedding and an appropriate comparison metric. Specifically, these methods provide a simple way to recognize novel data by calculating distance metrics between the image features and prototypes of each class. Augmentation based methods~\cite{wang2018low, schwartz2018delta, kim2020model, yang2021free} directly {alleviate} the dataset gap problems in few-shot learning by the way of producing various kinds of data. Typically, these augmentation methods utilize the generative model to create new training samples based on available ones and random noise or using the augmented data in the testing phase, and trying to learn a more robust new classifier for each task. Essentially, our PatchMix can still be categorized as an augmentation method. However, our method is different from those, as we are the \emph{first} to leverage augmentation in learning causal features by removing spurious correlation caused by sample selection bias, in the few-shot learning scenario. As we will show later, such a disentanglement of causal features can decrease the variance of inner-class instances' features, leading to easier classification among different classes.

Note that some of the former FSL methods like MTL~\cite{sun2019meta} utilize the strategy of hard sample mining, by taking the classes with low accuracy as the hard class. In contrast, our hardness-aware PatchMix encourages to 
mix each class with the most similar class, such that the generated images can have more distracting information to help improve the model robustness. Besides, by generating images from similar classes, the learned causal features can be more discriminative for classifying between these classes that are too similar (in terms of their causal features, \emph{e.g.}, to classify the dog and the wolf, although the learned causal features can concentrate on their shape, they are similar to each other) to be classified. 

\noindent\textbf{Data Augmentation.}  In the vanilla supervised learning, data augmentation is a widely used technique to facilitate training the deep models. 
The na\"ive augmentation strategies such as random flip and random crop serves as a data regularization, and have been applied to many practical topics and learning tasks~\cite{he2016deep,huang2017densely}.
Recently many works~\cite{zhang2017mixup, yun2019cutmix, michaelis2019benchmarking, gong2021person, chen2020gridmask, cubuk2020randaugment, ghiasi2021simple} are proposed by { mixing or substituting content between images}. Specifically, Mixup~\cite{zhang2017mixup} {mixes} two alternative images and the corresponding label. CutMix~\cite{yun2019cutmix} proposes to directly replace a randomly {selected} area. 
{Empirically, their efficacy has been evaluated on both fully-supervised and semi-supervised classification.}
In FSL, some works \cite{mangla2020charting} {attempt} to utilize manifold Mixup~\cite{verma2018manifold} to enhance the pretraining process. However, these works only directly apply existing data augmentation methods to FSL. As we will discussion in the following context, our PatchMix, instead of data regularization, is exclusively effective on removing spurious correlation for FSL, which cannot be realized by other data augmentation methods, as shown in the comparison between CutMix and our PatchMix in Sec.~\ref{sec:anaylysis} in terms of the identification of causal features. 

\noindent\textbf{Causal Learning.} Due to the invariance property of causal relation, there is an increasing attention paid at the intersection between causal inference and machine learning, see \cite{arjovsky2019invariant, scholkopf2021toward, sun2021recovering, liu2021learning, peters2016causal,rothenhausler2019causal} for out-of-domain generalization by removing the confounding bias. These works target on learning causal semantic features for better generalization, in the framework of \emph{Structural Causal Model} (SCM) pioneered by Judea Pearl \cite{pearl2009causality, pearl2000models}. Although it is common for FSL to have distributional shifts between training and test data due to sampling bias 
in data collection, few attempts have been made to address this issue. The work that is most related to us is IFSL \cite{yue2020interventional} that proposed to remove the bias from pre-trained knowledge by considering an intervened predictor, in the scenario when the pre-training step is adopted. Our departure here, is leveraging SCM to explicitly model the spurious correlation between causal and non-causal features, which can happen due to sample selection bias even  without the pre-training stage. With such a modeling, we theoretically show that our PatchMix is guaranteed to identify only the causal semantic features during learning. The empirical comparison of ours with IFSL are in Sec.~\ref{sec.ablation}. 

\section{Methodology}

\noindent \textbf{Problem Formulation.}
We formulate few-shot learning in the meta-learning paradigm. Particularly, the FSL model is learned via the episodes. {\color{black}{The episode should imitate the few-shot learning task: few support and query instances are sampled from several categories to train/evaluate the embedding model; the sampled support set is fed to the learner to produce a classifier, and then the loss and accuracy computed on the sampled query set is used respectively in training and testing phase.}} In general,
we have two sets of data, namely meta-train set $\mathcal{D}_s=\left\{ \left(\mathbf{I}_{i},y_{i}\right),y_{i}\in\mathcal{C}_{s}\right\}$ and meta-test set $\mathcal{D}_t=\left\{ \left(\mathbf{I}_{i},y_{i}\right),y_{i}\in\mathcal{C}_{t}\right\}$  corresponding to the base and novel dataset, individually. $\mathcal{C}_s$ and $\mathcal{C}_t$  ($\mathcal{C}_s \cap \mathcal{C}_t = \emptyset $) represent base and novel category sets respectively. The goal of FSL is to train a model on $\mathcal{D}_s$ which is well generalized to $\mathcal{D}_t$. As the definition of FSL task, the model can learn from few (\emph{e.g.}, one or five) labelled data from each category of $\mathcal{C}_t$.

We follow the former methods \cite{snell2017prototypical,sung2018learning} to adopt an $N$-way $K$-shot meta-learning strategy. Here $N$ denotes the number of categories in one episode and $K$ stands for {the} number of samples for each category in support set. 
Specifically, for each episode $\mathcal{T}$, $N$ categories are randomly sampled from $\mathcal{C}_s$ for training and $\mathcal{C}_t$ for testing, $K$ instances each for these selected categories to construct a support set $\mathcal{S}=\left\{ \left(\mathbf{I}_{i}^{\mathrm{supp}},y_{i}^{\mathrm{supp}}\right)\right\}$. 
Similarly we sample $M$ query samples per category, and thus construct the query set $\mathcal{Q}=\left\{ \left(\mathbf{I}_{i}^{\mathrm{q}},y_{i}^{\mathrm{q}}\right)\right\}$, and $\mathcal{S} \cap \mathcal{Q} = \emptyset$. Then the episode can be represented as $\mathcal{T}=\left\{\mathcal{S}, \mathcal{Q}\right\}$. In total each episode has $NK$ support images and $NM$ query images. {\color{black}{Note that while some methods, \emph{e.g.} \cite{snell2017prototypical} take different shot number during training and testing, we keep $K$ the same when training and evaluating our model.}}

This section is organized as follows. We first introduce a base model called DProto in Sec.~\ref{sec:baseline} to help define our PatchMix. Then we introduce the whole pipeline of our model which is overviewed in Fig.~\ref{fig:framework}. In particular, two stages of training are involved. For each stage, query image patches from training episodes are exchanged via our proposed PatchMix (Sec.~\ref{sec:patchmix}) to identify causal features (Sec.~\ref{sec:anaylysis}). The features of mixed images are then used in few-shot classification and reconstruction using the correlation-guided reconstruction (CGR) module (Sec.~\ref{sec:cgr}). In the second stage, each training episode is further enhanced with hardness-aware PatchMix (Sec.~\ref{sec:hard}), in which images from similar categories are mixed to induce more hardness for learning more discriminative features.

\subsection{A Base Model by Prototypes {\label{sec:baseline}}}
We introduce a simple model derived from ProtoNet~\cite{snell2017prototypical}. Specifically, given a few-shot episode $\mathcal{T}$, we first utilize a feature extractor network $\phi$ to {obtain} the feature maps of all images in the episode as $X=\phi(\mathbf{I}), \mathbf{I}\in\mathcal{T}$, where $X\in\mathbb{R}^{c\times h\times w}$. Then the prototype of the $i$-th class is calculated by taking the spatial-wise average of all support features belonging to this category, {followed by the} sample-wise average:
\begin{equation}
p_i=\frac{1}{K}\sum_{j=1}^{NK}\bar{X}_j^{\mathrm{supp}}\cdot \mathbbm{1}(y_j^{\mathrm{supp}}=i), i=1,\cdots, N
\label{eq:proto}
\end{equation}
\noindent where $\bar{X}_j^{\mathrm{supp}} \in \mathbb{R}^{c}$ is the spatial mean of $X_j^{\mathrm{supp}}$. For each query feature map $X^{\mathrm{q}}$, we build its prediction confidence map $\hat{X}^{\mathrm{q}}\in\mathbb{R}^{N\times h\times w}$ in the following way. For $i$-th class and spatial position indexed by $s,t$, $\hat{X}^{\mathrm{q}}_{i,s,t} = \frac{<X^{\mathrm{q}}_{:,s,t}, p_i>}{\|X^{\mathrm{q}}_{:,s,t}\|\|p_i\|}$, which is the normalized cosine distance between the prototype of $i$-th class and the query feature vector in this position. Meanwhile, a global classifier $f_{gc}$ consisting of a 1D convolutional layer is applied to $X^{\mathrm{q}}$ to get a prediction map.

This baseline model is a ProtoNet modified in two aspects: (1) We guide the model with patch-level labels instead of image-level ones to bring the stronger supervision. (2) We follow the former works \cite{hou2019cross} to add a global classifier, {\color{black}{which is used in training phase to enhance the supervision.}} Note that some existing works propose other kinds of modifications on ProtoNet such as learnable normalization \cite{oreshkin2018tadam}, attention modules \cite{ye2020feat} and more sophisticated distance metrics \cite{zhang2020deepemd}. We do not use these terms so that our model is simple enough to highlight the effect of our proposed PatchMix. We refer to this baseline model as DProto in the rest of this paper.

\begin{figure*}
    \centering
    \includegraphics[scale=0.43]{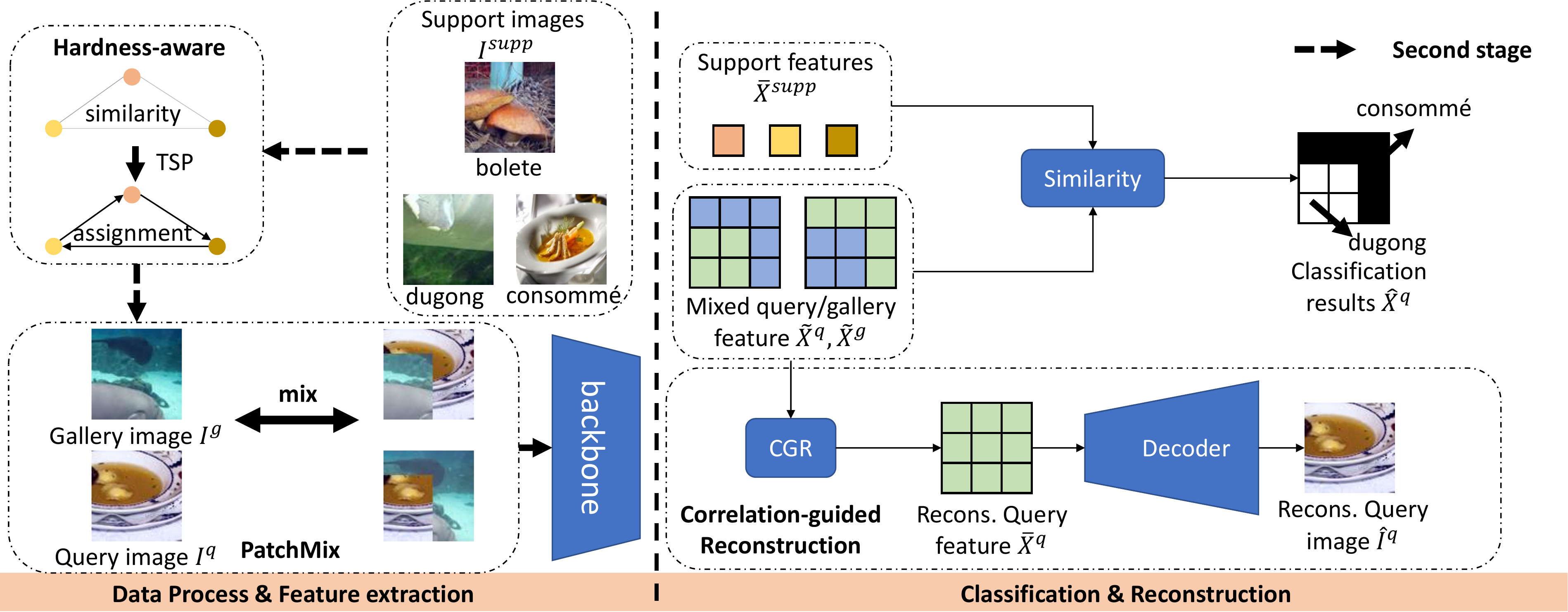}
    \caption{The training framework of our proposed model on a 5-way 1-shot task. In the first stage of training, for each query image we sample a gallery image whose random patch is then inserted into the corresponding position of query image. After feature extraction the classification result is achieved by comparing the feature vector of each position of query feature map to the averaged support feature vector, from which the classification loss is calculated together with the mixed label map. Meanwhile the mixed query and gallery feature map are processed with the correlation-guided module to reorganize them into one feature map that can restore the original query image. In the second stage, the PatchMix is further enhance with the hardness-aware module which controls the difficulty of mixture based on the distance among classes.}
    \label{fig:framework}
\end{figure*}

\noindent\textbf{Training.} For each input few-shot pair with inputs and outputs $\{X^{q}, y^{q}, \hat{X}^{q}\}$ ($y_q$ is a one-hot encoded vector for $|\mathcal{C}_s|$ classes), the objective function can be written as follow,
\begin{align}
    \mathcal{L}&=\ell_{f}+\frac{1}{2} \ell_{g}  \label{eq:loss}\\
     \ell_{g} &= -(\mathrm{log} \textrm{softmax}(f_{gc}(X^{\mathrm{q}})))^Ty^{q} \\
    \ell_{f} &= \frac{1}{hw}\sum_{s,t}\mathrm{log}\frac{e^{-\hat{X}^{\mathrm{q}}_{y^{\mathrm{q}}, s,t}}}{\sum_{i=1}^Ne^{-\hat{X}^{\mathrm{q}}_{i, s,t}}}
\end{align}
where $\ell_{f}$ is for the few-shot $N$-way classification, and $\ell_{g}$ is for the global many-shot $|\mathcal{C}_s|$-way (\textit{e.g.}, 64 for \textit{mini}ImageNet) classification. The feature extractor network $\phi$ is optimized based on $\mathcal{L}$ across randomly sampled episodes from base data.

\noindent\textbf{Testing.} {\color{black}{In the testing phase, only the feature extractor $\phi$ is used and the global classifier $f_{gc}$ is not engaged in testing.}} For each query image in an testing episode, we first get the confidence map $\hat{X}^{q}$ following the above method, then the image-level prediction is achieved by spatially averaging the confidence map.

\subsection{Image Augmentation by PatchMix \label{sec:patchmix}}
In this part we formally present PatchMix as an augmentation method tailored for few-shot learning. Concretely, for each query image $\{\mathbf{I}^{\mathrm{q}}, y^{\mathrm{q}}\}$, we randomly sample another query image as the gallery image $\{\mathbf{I}^g, y^g\}$ from which we collect the information to {switch}. Next we follow \cite{yun2019cutmix} to randomly select a box with width and height $(\hat{w},\hat{h})$ sampled from 
\begin{align}
    \lambda &\sim \mathrm{Unif}(0, 1) \\
    \hat{w} &= W\sqrt{1-\lambda}, \hat{h} = H\sqrt{1-\lambda},
\end{align}
\noindent where $W$, and $H$ are the width and height of the images. The center coordinate $(c_w,c_h)$ sampled from $c_w \sim \mathrm{Unif}(\lceil{\hat{w}/2}\rceil, W-\lceil{\hat{w}/2}\rceil)$, $c_h \sim \mathrm{Unif}(\lceil{\hat{h}/2}\rceil, H-\lceil{\hat{h}/2}\rceil)$. Denote $w_1,w_2,h_1,h_2$ as the left, right, lower and upper boundary of the box $(c_w,c_h,\hat{w},\hat{h})$. Then we can generate the 

the mask $M$ and the mixed image $\tilde{\mathbf{I}}^{\mathrm{q}}$ as:
\begin{align}
\label{eq:patchmix}
    M_{i,j} &= \left\{\begin{array}{rcl}
1 & & {w_1 \leq i \leq w_2, h_1 \leq j \leq h_2}\\
0 & & {o.w.}\\
\end{array} \right. \\
    \tilde{\mathbf{I}}^{\mathrm{q}} &= M \odot \mathbf{I}^g + (1-M) \odot \mathbf{I}^{\mathrm{q}}
\end{align}

In comparison with CutMix~\cite{yun2019cutmix} which adopts an image-level soft label as the ground truth, we still keep the patch-level hard labels to be the supervision information. We first interpolate the selected box to size of $w\times h$ as 
\begin{align}
    w_1^{'}, w_2^{'} &= \frac{w}{W}w_1, \frac{w}{W}w_2 \\
    h_1^{'}, h_2^{'} &= \frac{h}{H}h_1, \frac{h}{H}h_2
\end{align}
Then the new label map is set as
\begin{equation}
    Y_{i,j} = \left\{\begin{array}{rcl}
y^g & & {w_1^{'} \leq i \leq w_2^{'}, h_1^{'} \leq j \leq h_2^{'}}\\
y_q & & {o.w.}\\
\end{array} \right.
\end{equation}
Finally the mixed image $\tilde{\mathbf{I}}^{\mathrm{q}}$ is fed into the network, classified {by method} mentioned in Sec.~\ref{sec:baseline} and guided by the label map $Y$. In the subsequent section, we will introduce why such a simple operation can help learn causal features for better generalization. 

\subsection{Causal Explanation of PatchMix {\label{sec:anaylysis}}}

In this section, we explain the effectiveness of PatchMix in removing the spurious correlation from the learned representation, in the framework of \emph{Structural Causal Model} (SCM) \cite{pearl2009causality}. To describe the causal relations over all the variables, the SCM incorporates a directed acyclic graph (DAG) $G=(\mathcal{V},\mathcal{E})$, such that $X \to Y \in \mathcal{E}$ means $X$ has a direct causal effect to $Y$. Due to its ability to encode priories  beyond data, it has been increasingly leveraged to disentangle the \emph{causal semantic features} \cite{sun2021recovering, liu2021learning} from other features for out-of-distribution generalization. Inspired by such a spirit, we provide an ad-hoc analysis to explain how the PatchMix operator can remove the spurious correlation and only keep causal features for prediction.

\begin{figure}[ht!]
    \centering
    \includegraphics[scale=0.45]{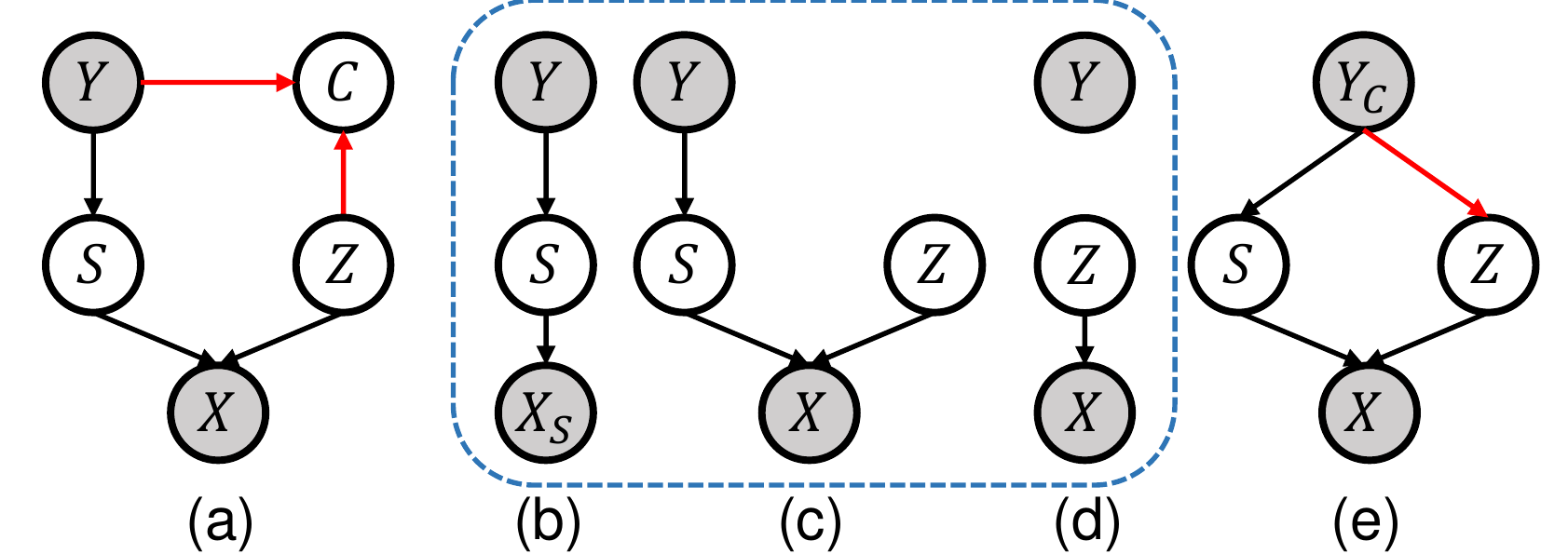}
    \caption{Causal graph of models (a) with sample selection bias, (b) using PatchMix for patches consisting of single image, (c) using PatchMix for patches consisting of two images, (d) using Patchmix for patches whose causal features are completely removed and (e) using CutMix. The $S, Z$ respectively denote the causal and non-causal features. $Y$ indicates the class label. $C$ is an indicator for train/test set. Red arrows mean the resource of spurious correlation between class label and non-causal features.
    The patch image with PatchMix can generated either from single image or two images, as respectively marked by the dot blue boxes. For both types, the PatchMix can remove the dependency from sampling bias. The CutMix replaces the original label $Y$ with an interpolated one $Y_c$; hence it is determined by features from both images.}
    \label{fig:causal-graph}
\end{figure}

The data collection procedure can often incur sampling bias, inducing the dependency among some originally independent concepts, \emph{e.g.}, the dog is associated with the grass more often than with the river in the collected dataset. This can cause the \emph{spurious} correlation between causal features and non-causal ones, leading to non-causal features learned during training, which can hurt the performance. 

We formulate this spurious correlation in Fig.~\ref{fig:causal-graph} (a), in which $S, Z$ respectively denote the causal (\emph{i.e.}, the texture, shape of the dog) and non-causal features (\emph{i.e.}, features related to other objects such as people) in terms of its relation with the outcome $Y$. The label $Y$ and non-causal features $Z$ altogether generates the sampling procedure denoted as $C$, \emph{e.g.}, the sampler collected the dog (\emph{i.e.}, outcome $Y$) on the grass (\emph{i.e.}, $Z$) in some communities when people walked with their pets. Here, the $C$ with $C=1$ (or $C=0$) means that the corresponding sample is in the training (or test) set. As $C$ is the collider on the path $S \gets Y \to C \gets Z$, it induces the correlation between $S$ and $Z$, \emph{i.e.}, $S \not\perp Z|C$. Moreover, for any sample $I$, the $Y \not \perp Z|S,I,C=1$, making the non-causal features learned during training process. However, it may fail to generalize as $P(S,Z|C=1)$ may not equal to $P(S,Z|C=0)$. 

On the other hand, as shown in Fig.~\ref{fig:causal-graph} (b),(c),(d),  we may have three cases of patch images after PatchMix operation:
 \textbf{i)} Figure~\ref{fig:causal-graph}~(b)  contains only causal features; \textbf{ii)} in Fig.~\ref{fig:causal-graph}~(c),  part of both causal and non-causal features are replaced by others; \textbf{iii)} in Fig.~\ref{fig:causal-graph}~(d), all causal information is removed. 

For those patches in $\mathbf{I}^q$ that contain causal information, the case described in Fig.~\ref{fig:causal-graph}~(b) can naturally decrease the influence of non-causal features. Thus such a case is 
 less correlated with the spurious features, as it directly removes features from other patches that may be spurious. For the case described in Fig.~\ref{fig:causal-graph}~(c), the generated patch integrates $S$ and $Z$ randomly from two different images, breaking the original sampling procedure and thus the selection bias it induces. In this regard, the causal features $S$ no longer depends on $Z$. Formally speaking, we can have $Y \perp Z|S,I$ for these patches in the training set, making it possible to eliminate the non-causal features if the model is trained well.

Note that our Patchmix can also handle the extreme cases that some patches of the "causal features" are completely removed, which corresponds to the graph in Fig.~\ref{fig:causal-graph} (d).
Theoretically our PatchMix can still  reduce the correlation between spurious features and causal features. Specifically, equipped with PatchMix, these patches also randomly mix information from other patches. In this regard, the patterns of original non-causal features are broken, and are randomly replaced by others, making correlations between these non-causal features and $Y$ no longer exist. In this way, compared to vanilla patch-level training methods without any mixing strategies (\emph{i.e.}, base model) or with other mixing strategies, our PatchMix can enforce the model to pay more attention to causal features, and remove spurious correlation of other non-causal features. This result can be verified by a noticeable improvement of our method over the base models in Tab.~\ref{tab:main_mini} in Tab.~\ref{tab:ablation}, additionally with interpretable visualization results in Fig.~\ref{fig:causal_visual}.

\begin{theorem}[Disentangling Causal Features]
\label{thm:disentangle}
Suppose the neural network parameter is composed of two parts: $(\psi,\theta)$, where $\psi: \mathcal{I} \to \mathbb{R}^{\mathrm{dim}(S)}$ extracts a representation $\psi(I)$ from the image, followed by $\theta$ for predicting the label. We denote $p^{o}(\psi,\theta)(y|I)$ and $p^{\mathrm{patch}}(\psi,\theta)(y|I)$ as the trained model on the original data (with Fig.~\ref{fig:causal-graph} (a)) and the one on the data after PatchMix (with Fig.~\ref{fig:causal-graph} (b),(c),(d)). Assume that the structural equation for the image $I$ is injective, \emph{i.e.}, $I \gets f_I(S,Z)$, then if both $p^{o}(\psi,\theta)(y|I)$ and $p^{\mathrm{patch}}(\psi,\theta)(y|I)$ are trained to perfectly fit the ground-truth distribution $p^*(y|I)$, we have:
\begin{itemize}
    \item In $p^{o}(\psi,\theta)(y|I)$, the $\psi^o(I)$ is dependent on the non-causal features $Z$.
    \item In $p^{\mathrm{patch}}(\psi,\theta)(y|I)$, it is with Lebesgue measure 0 for $\psi^{\mathrm{patch}}(I)$ to depend on $Z$. 
\end{itemize}
\end{theorem}

\begin{remark}
The injective assumption on $f_I$ is widely assumed in the literature, such as (variational) auto-encoder \cite{khemakhem2020variational} since it has been empirically verified that the image can be recovered perfectly from its latent embedding; causal inference \cite{janzing2009identifying, peters2014causal} for identifiability consideration and few-shot learning \cite{teshima2020few}. 
\end{remark}

\begin{proof}

For both $p^{o}(\psi,\theta)(y|I)$ and $p^{\mathrm{patch}}(\psi,\theta)(y|I)$, we have $Y \perp I|\psi(I)$ since $\psi(I)$ blocks all paths in the neural network from the input $I$ to the output $Y$, as the predicted label is the same to the ground-truth label $Y$ since both $p^{o}(\psi,\theta)(y|I)$ and $p^{\mathrm{patch}}(\psi,\theta)(y|I)$ equal to $p^*(y|I)$. According to Fig.~\ref{fig:causal-graph} (a), we have $Y \not \perp I|S$, therefore if $\psi^o(I)$ does not depend on $Z$, \emph{i.e.}, $\psi^o(I)=h(S)$ for some $h$, it fails to make $Y$ and $I$ conditionally independent. On the other hand, it is sufficient for $\psi^{\mathrm{patch}}(I)$ to depend only on $S$ to make $Y \perp I|\psi^{\mathrm{patch}}(I)$. Suppose $\psi^{\mathrm{patch}}(I)=h(S,Z)$ for some $h$. For the patch from a single whole image (marked by dot blue box in Fig.~\ref{fig:causal-graph} (b)), the representation can only contain the causal features $S$ as it removes other spurious correlated features. To show $h$ is independent to $Z$ for the patch as the mixture of two images (marked by Fig.~\ref{fig:causal-graph} (c)), recall that $p^{\mathrm{patch}}_{\theta}(y|\psi^{\mathrm{patch}}(I)) = p^*(y|[f_I^{-1}]_{\mathcal{S}}(I)) = p^*(y|s)$ with $I = f_I(s,z)$ and the injective assumption for $f_I$, we have that for each $s$ and two different $z_1,z_2$, we have $p^{\mathrm{patch}}_\theta(y|h(s,z_1)) = p^{\mathrm{patch}}_\theta(y|h(s,z_2)) = p^*(y|s)$, which means that the $p^{\mathrm{patch}}_\theta(y|h(s,z_1))$ does not depend on $z$. Since it is with Lebesgue measure 0 for the parameter $\theta$ to be 0 on $h(S,Z)$, which means an unblocked path from $Z$ to $Y$ if $h$ depends on $Z$, which comes up a contradiction. Therefore, it is with Lebesgue measure 0 to let $h$ depend on the non-causal feature $Z$. For Fig.~\ref{fig:causal-graph} (d), as original non-causal features $Z$ are randomly mixed with other features, the correlation between $Z$ and $Y$ is thus broken.
\end{proof}

This conclusion means it generically holds for the learned representation to eliminate the information of non-causal features. This can explain the benefit of PatchMix in removing the non-causal features during learning, by breaking the dependency between the causal features and the non-causal ones.

\noindent \textbf{Comparison with CutMix.} The difference of our method with CutMix is that the latter adopts a soft label (namely $Y_c$) to integrated images (as shown in Fig.~\ref{fig:causal-graph} (e)), which is determined by the label from both images. Therefore, such a soft label $Y_C$ is related to both features from two different images. In this regard, $Y \not\perp Z|S,I$ for each $I$, leading to the non-disentanglement between $S$ and $Z$. Therefore, this learning mechanism is not endowed with disentanglement ability that can be helpful for transferring to novel categories, as summarized in the following: 
\begin{theorem}
\label{thm:cutmix}
Denote $p^{\mathrm{cut}}(\psi,\theta)(y|I)$ as the trained model on the data after CutMix (with Fig.~\ref{fig:causal-graph} (d)). Under the same injective assumption on $f_I$, we have that the $\psi^{\mathrm{cut}}(I)$ is dependent on features $Z$, if $p^{\mathrm{cut}}(\psi,\theta)(y|I)$ is trained to equal to the ground-truth distribution $p^*(y|I)$.
\end{theorem}
\begin{proof}
Similar to the proof of Thm.~\ref{thm:disentangle}, we have $Y_c \perp I|\psi^{\mathrm{cut}}(I)$, since in Fig.~\ref{fig:causal-graph} (d), we have $Y \not\perp Z|S,I$. Thus, it is necessary for $\psi^{\mathrm{cut}}(I)$ to depend on both $S$ and $Z$. Further, if the spurious correlation is strong enough, it is incapable to disentangle $S$ from $Z$ as they play nearly symmetric roles in affecting $Y$ and generating the input $I$.
\end{proof}

\noindent \textbf{Relationship with neural collapse.} As introduced in Sec.~\ref{sec:intro}, a recent work~\cite{galanti2021role}  studies the 
neural collapse that the variance of features from the novel category is small enough to separate different categories; and it  
claims that the property of neural collapse 
can be well transferred to novel data given sufficient base training data. Intuitively, the learned disentangled causal features can better correlate with novel class labels compared with the non-causal features. Therefore, the model trained with PatchMix is expected to receive better neural collapse on novel data since PatchMix can help break up the dependency between causal and non-causal features during training.

\begin{figure}
    \centering
    \includegraphics[scale=0.37]{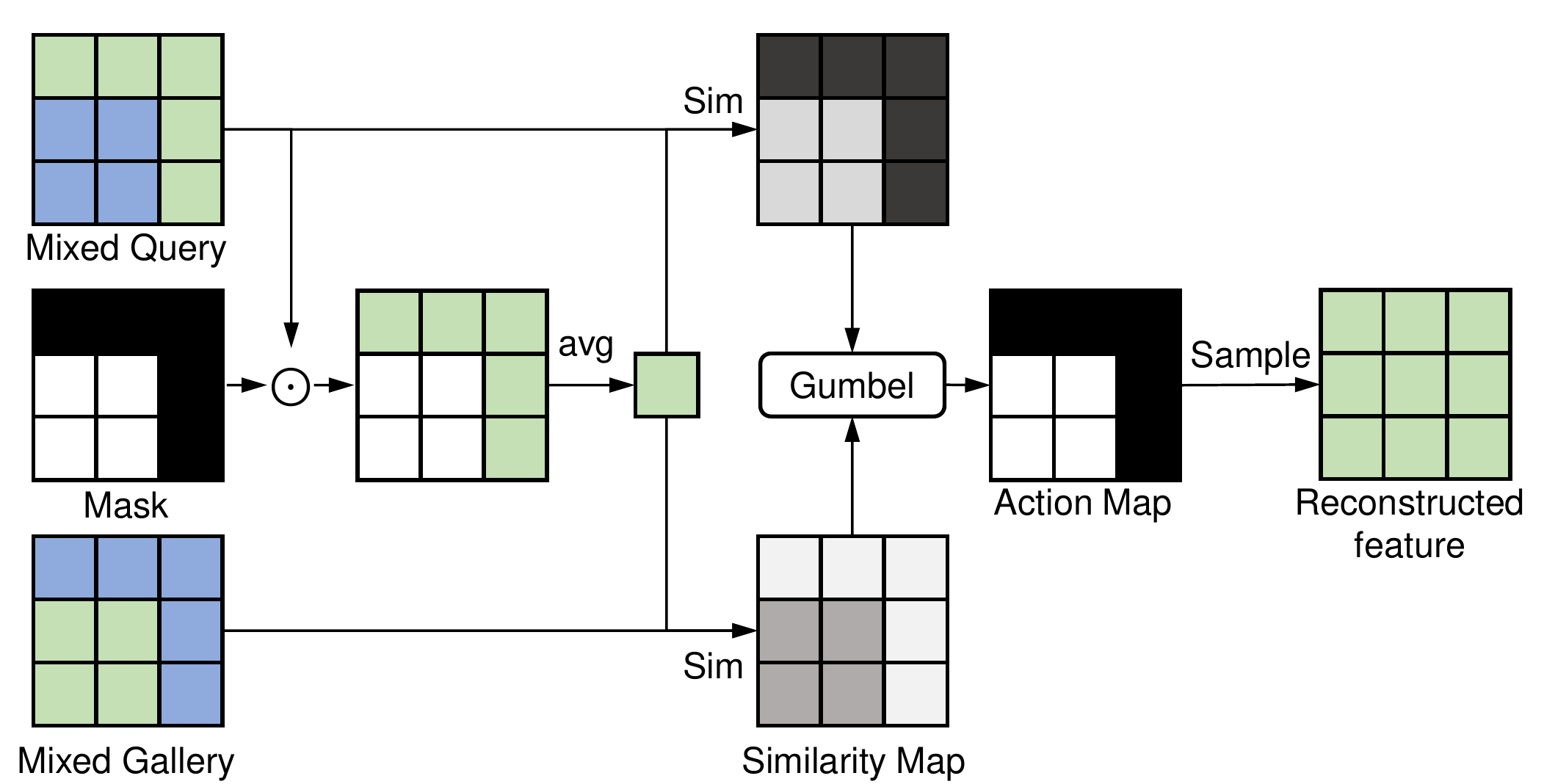}
    \caption{The module for correlation-guided reconstruction. Each position of the mixed query and gallery feature map is compared with the known area from the original image. Then the similarity is further normalized with gumbel-softmax to sample an action map that is used to select the patches.}
    \label{fig:corr}
\end{figure}

\subsection{Discriminative Feature Enhancement}
Although the non-causal features can be eliminated, it is not necessarily for these learned causal features to be discriminative enough for classification between similar classes. For example, the shape feature is causally related to both the dog and the wolf; however, the learned feature from two classes may be too similar to be classified. To make our causal features be more discriminative,
we further propose two modules, i.e., \emph{correlation-guided reconstruction} module in Sec.~\ref{sec:cgr} and \emph{hardness-aware PatchMix} module in Sec.~\ref{sec:hard}, respectively for instance discrimination and classifying between similar classes.

\subsubsection{Correlation-guided Reconstruction Module \label{sec:cgr}}
Inspired by recent works on self-supervised learning \cite{he2020momentum}, we propose a novel module while takes advantage of PatchMix procedure to contrast image features from different categories. Specifically, given query image $\mathbf{I}^{\mathrm{q}}$, gallery image $\mathbf{I}^{\mathrm{g}}$ and the corresponding mixture mask $M$, we can have the resulting mixed image $\tilde{\mathbf{I}}^{\mathrm{q}}$ and its feature map $\tilde{X}^{\mathrm{q}}$. Meanwhile, by exchanging the role of $\mathbf{I}^{\mathrm{q}}$ and $\mathbf{I}^{\mathrm{g}}$ and using $1-M$ as mask, we can have image $\tilde{\mathbf{I}}^g$ and feature map $\tilde{X}^g$ which contain the counterpart of $\tilde{\mathbf{I}}^{\mathrm{q}}$.
{\color{black}{In this way, $\mathbf{I}^{\mathrm{q}}$ and $\mathbf{I}^{\mathrm{g}}$ can be properly reconstructed by selecting  correct patches from these two feature maps. The selection process can be based on successfully telling apart patch-level features from different images.}} Therefore, we make the reconstruction of both of query images and gallery images as the goal of this module.

Take the $(i,j)-$th patch in $\tilde{\mathbf{I}}^{\mathrm{q}}$, we evaluate the confidence of it belonging to $\mathbf{I}^{\mathrm{q}}$ and $\mathbf{I}^{\mathrm{g}}$ as 
\begin{align}
\alpha^{\mathrm{q}}_{i,j}&=\sum_{m,n=1,1}^{h,w}\sum_{m,n\neq i,j} M_{m,n}\cdot <\tilde{f}_{m,n}^{\mathrm{q}}, \tilde{f}_{i,j}^{\mathrm{q}}> \\   
\alpha^g_{i,j}&=\sum_{m,n=1,1}^{h,w}\sum_{m,n\neq i,j} M_{m,n}\cdot <\tilde{f}_{m,n}^{\mathrm{g}}, \tilde{f}_{i,j}^q>
\end{align}
In other words, we compare each patch to the known patches from $\mathbf{I}^{\mathrm{q}}$ except itself. The confidence for belonging to class $B$ can be calculated in the same way by replacing $M$ and $\tilde{f}_{m,n}^{\mathrm{q}}$ with $1-M$ and $\tilde{f}_{m,n}^g$, respectively. Then we can select these patches according to the confidence,
\begin{align}
    \bar{X}^{\mathrm{q}}_{i,j}&=\hat{\alpha}^g_{i,j}\tilde{X}_{i,j}^{\mathrm{g}} + \hat{\alpha}^{\mathrm{q}}_{i,j}\tilde{X}_{i,j}^q \\
    \hat{\alpha}^g_{i,j} &= \sigma(\alpha^{\mathrm{q}}_{i,j}/T) \\
    \hat{\alpha}^{\mathrm{q}}_{i,j} &= \sigma(\alpha^{\mathrm{q}}_{i,j}/T)
\end{align}
where $T$ is temperature and $\sigma$ is a normalization function implemented as gumbel-softmax~\cite{jang2016gumbel}. 
After obtaining the merged feature $\bar{X}^{\mathrm{q}}, \bar{X}^g$, we process them with a decoder, whose structure is a reversal of the feature extractor, to generate the reconstructed images $\bar{\mathbf{I}}^{\mathrm{q}}, \bar{\mathbf{I}}^g$. 

The objective function of this module can be formulated as
\begin{align}
    L_{CR} &= L_{sel} + \lambda L_{rec} \\
    L_{sel} &= \mathrm{CE}(\hat{\alpha}^{\mathrm{q}}, M) + \mathrm{CE}(\hat{\alpha}^g, 1-M) \\
    L_{rec} &= \|\hat{A}-A\|_1 + \|\hat{B}-B\|_1
\end{align}

It is noteworthy that unlike the common way in conditional image generation, our CGR does not directly concatenate $\bar{X}^{\mathrm{q}}$ and $\bar{X}^g$ as input. Also, we do not adopt techniques like adaptive normalization~\cite{huang2017adain}. Compared with reconstruction-based method, which was empirically shown \cite{robert2018hybridnet} to be not helpful for classification, we additionally separate the features from different classes by successfully distinguish query images and their gallery counterpart which have different labels. Without the correlation guidance, the task can easily degenerate to the task of repairing images \emph{e.g.}, image inpainting, which unfortunately may not be necessarily helpful to few-shot classification.

\subsubsection{Hardness-aware PatchMix Module \label{sec:hard}}

Since all query images are mixed with random sampled images from the sample episode, the vanilla PatchMix lacks control of difficulty, thus it may be even difficult for learned causal features to be discriminative enough, especially between similar classes. To resolve this problem, we propose mixing images from similar classes, in which the similarity among base categories can be estimated by a coarse prediction model with PatchMix in the first stage.

Specifically, given an episode with support set $\mathcal{S}$ and a trained feature extractor $\phi$, we can first get the prototype of each category $p_i, i=1,\cdots, N$ in the same way as Eq.~\ref{eq:proto}. Then the similarity between $i$-th class and $j$-th class can be represented by the normalized cosine similarity between their prototypes:
\begin{equation}
    \mathrm{sim}(i,j)=\frac{<p_i, p_j>}{\|p_i\|\|p_j\|}
    \label{eq:sim}
\end{equation}
In this way, we can treat the relationship among the $N$ categories as a complete $N$-node graph whose edges are the negative similarity of the corresponding prototypes. Then we solve a Travelling Salesman
Problem (TSP)  by setting a random class as the starting point, which results in a path with the lowest cost, thus can generate images from classes that are easy to be obfuscated owing to high similarity. Each node is assigned as the gallery class to its parent node, followed by the PatchMix training scheme as introduced before. The objective function of hardness-aware PatchMix can be written as
\begin{align}
\label{eq:second_loss}
    \mathcal{L}_{ha} &= \mathcal{L} + \ell_{kd} \\
    \ell_{kd} &= \frac{1}{NQ} \sum_{i=1}^{NQ} \psi(F(\mathbf{I}_i^{\mathrm{q}}), \hat{F}(\mathbf{I}_i^{\mathrm{q}}))
\end{align}
where $\mathcal{L}$ is defined in Eq.~\ref{eq:loss} and $\psi$ is a knowledge distillation loss~\cite{hinton2015distilling}, which can be implemented as MSE or KL divergence.

In all, the workflow of our algorithm can be separated into two stage, as stated in Alg.~\ref{alg:main}, including training the baseline model in Sec.~\ref{sec:baseline} together with PatchMix and CGR in the first stage and hardness-aware PatchMix in the second stage.

\begin{algorithm}[t]
\caption{Training FSL Model with PatchMix}
\label{alg:main}
\begin{algorithmic}[1]
\REQUIRE $\mathcal{D}_s$: meta-train set
\STATE \# coarse PatchMix
\WHILE{not converge}
\STATE Sample batch of episodes $B$ from $\mathcal{D}_s$
  \FORALL{episode}
    \FORALL{query sample $\mathbf{I}_i^{\mathrm{q}}$}
      \STATE Apply PatchMix to $\mathbf{I}_i^{\mathrm{q}}$ according to Eq.~\ref{eq:patchmix}
      \ENDFOR
     \ENDFOR
 \STATE Calculate objective function $\mathcal{L}$ as in Eq.~\ref{eq:loss}
 \STATE Update model parameters according to $\mathcal{L}$
\ENDWHILE
\STATE
\STATE \# hardness-aware PatchMix
\WHILE{not converge}
\STATE Sample batch of episodes $B$ from $\mathcal{D}_s$
  \FORALL{episode}
    \STATE calculate class-wise similarity according to Eq.~\ref{eq:sim}
    \STATE assign gallery class
    \FORALL{query sample $\mathbf{I}_i^{\mathrm{q}}$}
      \STATE Apply PatchMix to $\mathbf{I}_i^{\mathrm{q}}$ according to Eq.~\ref{eq:patchmix}
      \ENDFOR
     \ENDFOR
 \STATE Calculate objective function $\mathcal{L}$ as in Eq.~\ref{eq:second_loss}
 \STATE Update model parameters according to $\mathcal{L}$
\ENDWHILE
\end{algorithmic}
\end{algorithm}

\subsection{Adaptation to Unsupervised FSL\label{ufsl}}
In few-shot learning, it is sometimes hard to get  {any} labeled samples, especially when the labeling cost is high. To resolve this problem, some recent works~\cite{hsu2018unsupervised,khodadadeh2018unsupervised} attempt to emphasize unsupervised few-shot learning. These works firstly adopt unsupervised learning methods to train the networks, and then the trained networks are used for testing. For instance, CACTUs~\cite{hsu2018unsupervised} selects several unsupervised learning methods such as DeepCluster~\cite{caron2018deep} to get a representation vector for each sample. These representation vectors are then utilized to cluster the samples so that each sample can be assigned by a pseudo label.
After obtaining the pseudo label for each sample, we can {apply} some few-shot learning methods to train the networks. 
 
Following the standard practice above, this section further extends PatchMix to unsupervised FSL. Specifically, we adopt CACTUs-ProtoNet as our baseline. There are two main components for CACTUs, \emph{unsupervised pretraining} and \emph{pseudo label training}. Our adaptation considers using PatchMix for both parts.
 
\noindent (1) \textbf{Unsupervised pretraining.} We take into account the current progress in unsupervised learning and employ a modified version of MoCo~\cite{he2020momentum} based on our PatchMix. For the original version of MoCo, a batch of sampled images are augmented to generate two sets of images called the key images and the query images.
 For each query image $\mathbf{I}^{\mathrm{q}}_i$, the key image $\mathbf{I}^{key}_i$ that shares the same image before augmentation is its positive sample and other key images are negative samples.
 These images are fed into networks to get spatially averaged feature vectors, in which vectors of the $i-$th key image are denoted as $\bar{X}_i^{key} \in\mathbb{R}^{N}$ and vectors of the $i-$th query image are denoted as $\bar{X}_i^{\mathrm{q}} \in\mathbb{R}^{N}$. Here we use  $\hat{X}^{\mathrm{q}}_{ij} = \frac{<\bar{X}_i^{q}, \bar{X}_j^{key}>}{\|\bar{X}_i^{q}\|\|\bar{X}_j^{key}\|}$ to denote the cosine similarity between two features.
 For each batch of query images, we calculate its loss with the form  
\begin{equation}
    \mathcal{L}_{MoCo} =\sum_{i} -\mathrm{log}\frac{e^{\hat{X}^{\mathrm{q}}_{ii}/T}}{\sum_{j} e^{\hat{X}^{\mathrm{q}}_{ij}/T}}
\end{equation}
where $T$ is the temperature. 

To implement our modified version which we call PatchMoCo, we first consider a dense variant of MoCo. We remove the global average pooling layer for query features and get the spatial feature $X_i^{\mathrm{q}}\in\mathbb{R}^{N\times h\times w}$. Each spatial position of $X_i^{\mathrm{q}}$ is denoted as $X_{i,s,t}^{\mathrm{q}}\in\mathbb{R}^{N}$.
Similarly the cosine similarity for each position $s,t$ is 
  $  \hat{X}^{\mathrm{q}}_{istj} = \frac{<{X}_{ist}^{q}, \bar{X}_j^{key}>}{\|{X}_{ist}^{q}\|\|\bar{X}_j^{key}\|}$.
  The altered loss has the following form,
  \begin{equation}
    \mathcal{L}_{d} =\sum_{i}\sum_{st} -\mathrm{log}\frac{e^{\hat{X}^{\mathrm{q}}_{isti}/T}}{\sum_{j} e^{\hat{X}^{\mathrm{q}}_{istj}/T}}
\end{equation}
  Afterwards, we add our PatchMix into the unsupervised pretraining.
  The process of mixing patches is  as the way in Sec.~\ref{sec:patchmix}. We modify the form of loss function.  If some of its patches are switched for each image, these new patches are negative samples. For these new patches, we cannot find their positive samples, so we do not include them into the calculation of final loss. 
  For convenience, we get a mask $K^{\mathrm{q}}_i\in\{0,1\}^{h\times w}$, where 1 means patch without switch and 0 means switched patch. The new loss function is,
  \begin{equation}
    \mathcal{L}_{PMC} =\sum_{i}\sum_{st} -K^{\mathrm{q}}_{ist}\mathrm{log}\frac{e^{\hat{X}^{\mathrm{q}}_{isti}/T}}{\sum_{j} e^{\hat{X}^{\mathrm{q}}_{istj}/T}}
\end{equation}

  \noindent (2) \textbf{Pseudo-label training}: After we get pseudo-label {by clustering the feature of unsupervised pretraining}, each training sample can be used as the same in supervised FSL. Then PatchMix is applied to the training phase as in Sec.~\ref{sec:patchmix}.

\section{Experiments}
\subsection{Datasets and setting}
\noindent\textbf{Datasets.} We mainly adopt four datasets for our experiments. \textbf{i)} \textit{mini}ImageNet dataset~\cite{vinyals2016matching}, containing 600 images in each of the 100 categories, is a small subset of ImageNet. We follow the split in \cite{ravi2016optimization}, where 64, 16, 20 {categories} are used for train,  validation and test {set}, respectively.  ii) \textit{tiered}ImageNet dataset~\cite{ren2018meta} is a larger subset of ILSVRC-12 dataset. It consists of 34 categories with 779,165 images in total. These categories are further broken into 608 {categories}, where 351 {categories} are used for training, 97 for validation and 160 for testing. \textbf{iii)} CIFAR-FS~\cite{bertinetto2018meta} divides CIFAR-100 into 64 meta-train {categories}, 16 meta-val {categories} and 20 meta-test {categories}. \textbf{iv)} CUB~\cite{welinder2010caltech}, a bird dataset with 200 total categories and 6033 total images. In few-shot learning, 100, 50, 50 species are respectively used for training, validation and test sets. Besides, Cars~\cite{krause2013car}, Places~\cite{zhou2017places} and Plantae~\cite{van2018plantae} are also utilized in cross-domain setting following \cite{tseng2020cross}. Images in all datasets are resized to $84\times 84$ before training and testing. 

\noindent\textbf{Experimental setup.} Stochastic Gradient Descent (SGD)~\cite{bottou2010large} with $5e-4$ weight decay and cosine learning rate decay \cite{loshchilov2016sgdr} is used to optimize our model. For \textit{mini}ImageNet and CIFAR-FS, the initial learning rate is set as $0.15$ and $0.05$ for \textit{tiered}ImageNet. Random cropping, horizontal flipping and color jittering are adopted for data augmentation during training, which is the same as in CAN~\cite{hou2019cross}.
We test 2000 episodes sampled from meta-test set for all experiments. For correlation-guided reconstruction, $\lambda_1$ is set as $0.5$ for all datasets, $\lambda_2=0.1$ for \textit{tiered}ImageNet and $0.25$ for other datasets. As for the hardness-aware PatchMix, we empirically set the distillation function $\psi$ as MSE for 1-shot tasks and KL divergence for 5-shot tasks. Models and codes will be released.

\subsection{Comparison with state-of-the-art methods}
To extensively show the effectiveness of our method, we test our PatchMix across three commonly-used settings, \emph{i.e.}, single domain, cross domain and unsupervised few-shot learning. For each setting, we compare our model with the recent state-of-the-art competitors, where the average accuracy is reported. For supervised learning settings, we additionally report the 95\% \emph{confidence interval} (CI). 

\begin{table}[t]
 \centering
 {
  \begin{tabular}{ l|l|cc}
  \hline
 \multirow{2}{*}{Model} & \multirow{2}{*}{Backbone} &  \multicolumn{2}{c}{\textit{mini}ImageNet}  \tabularnewline
 \cline{3-4}
& &  1-shot  &   5-shot\\

   \shline
   ProtoNet~\cite{snell2017prototypical} & \multirow{5}{*}{Conv4} & 49.42$\pm$0.78 & 68.20$\pm$0.72\tabularnewline
   \cline{3-4}
   MatchingNet~\cite{vinyals2016matching}&  & 43.56$\pm$0.84 & 55.31$\pm$0.73 \tabularnewline
   \cline{3-4}
   RelationNet~\cite{sung2018learning} &  & 50.44$\pm$0.82 & 65.32$\pm$0.70 \tabularnewline
   \cline{3-4}
   MAML~\cite{finn2017model} &  & 48.70$\pm$1.75 & 63.11$\pm$0.92 \tabularnewline
   \cline{3-4}
   Dynamic Few-shot~\cite{gidaris2018dynamic} &  & 56.20$\pm$0.86 & 72.81$\pm$0.62\tabularnewline
    \hline
   \cline{3-4}
     LEO~\cite{rusu2018meta} & \multirow{7}{*}{WRN-28} & 61.76$\pm$0.08 & 77.59$\pm$0.12 \tabularnewline
     \cline{3-4}
     PPA~\cite{qiao2018few} &  & 59.60$\pm$0.41 & 73.74$\pm$0.19\tabularnewline
     \cline{3-4}
     Robust dist++~\cite{dvornik2019diversity} &  & 63.28$\pm$0.62 & 81.17$\pm$0.43 \tabularnewline
     \cline{3-4}
     wDAE~\cite{gidaris2019generating} &  & 61.07$\pm$0.15 & 76.75$\pm$0.11\tabularnewline
 \cline{3-4}
       CC+rot~\cite{gidaris2019boosting} &  & 62.93$\pm$0.45 & 79.87$\pm$0.33 \tabularnewline \cline{3-4}
       DC~\cite{yang2021free}$^{*}$ &  & 67.96$\pm$0.45 & 83.45$\pm$0.31 \tabularnewline \cline{3-4}
    FEAT~\cite{ye2020feat} &  & 65.10$\pm$0.20 & 81.11$\pm$0.14 \tabularnewline
    \hline
     TapNet~\cite{yoon2019tapnet} & \multirow{14}{*}{Res-12} & 61.65$\pm$0.15 & 76.36$\pm$0.10 \tabularnewline
     \cline{3-4}
     MetaOptNet~\cite{lee2019meta} &  & 62.64$\pm$0.61 & 78.63$\pm$0.46 \tabularnewline
     \cline{3-4}
     CAN~\cite{hou2019cross} &  & 63.85$\pm$0.48 & 79.44$\pm$0.34 \tabularnewline
     \cline{3-4}
      FEAT~\cite{ye2020feat} &  & 66.78$\pm$0.20 & 82.05$\pm$0.14 \tabularnewline \cline{3-4}
      E$^3$BM~\cite{liu2020ensemble} &  & 63.80$\pm$0.40 & 80.10$\pm$0.30 \tabularnewline \cline{3-4}
      DSN-MR~\cite{simon2020adaptive} &  & 64.60$\pm$0.72 & 79.51$\pm$0.50 \tabularnewline \cline{3-4}
      Net-Cosine~\cite{liu2020negative} &  & 63.85$\pm$0.81 & 81.57$\pm$0.56 \tabularnewline\cline{3-4}
      FRN~\cite{wertheimer2021few} &  & 66.45$\pm$0.19 & 82.83$\pm$0.13 \tabularnewline \cline{3-4}
      Tian et.al.~\cite{tian2020rethinking} &  & 64.82$\pm$0.60 & 82.14$\pm$0.43\tabularnewline \cline{3-4}
      ConstNet~\cite{xu2021attentional} &  & 64.89$\pm$0.23 & 79.95$\pm$0.37 \tabularnewline \cline{3-4}
      IEPT~\cite{zhang2020iept} &  & 67.05$\pm$0.44 & 82.90$\pm$0.30 \tabularnewline \cline{3-4}
      MELR~\cite{fei2020melr} &  & 67.40$\pm$0.43 & 83.40$\pm$0.28 \tabularnewline \cline{3-4}
      DMF~\cite{xu2021learning} &  & 67.76$\pm$0.46 & 82.71$\pm$0.31 \tabularnewline \cline{3-4}
      DeepEMD~\cite{zhang2020deepemd} &  & 68.77$\pm$0.29 & 84.13$\pm$0.53 \tabularnewline \cline{1-4}
      Base Model & \multirow{3}{*}{Res-12} & 64.96$\pm$0.51 & \bf 80.51$\pm$0.33 \tabularnewline\cline{3-4}
     Ours &  &\bf 69.38$\pm$0.46 & \bf 84.14$\pm$0.30 \tabularnewline\cline{3-4}
     Ours+DC &  &\bf 69.75$\pm$0.44 & \bf 84.88$\pm$0.30  \tabularnewline
     \hline
  \end{tabular}
 }
 \vspace{0.1in}
\caption{\label{tab:main_mini} 5-way few-shot accuracies with $95\%$ confidence interval on \textit{mini}ImageNet, {*} denotes results reproduced by us.}
\end{table}

\begin{table}[t]
 \centering
 {\setlength{\tabcolsep}{1.3mm}{
  \begin{tabular}{ l|cccc}
  \hline
 \multirow{2}{*}{Model} &   \multicolumn{2}{c}{\textit{tiered}ImageNet}&   \multicolumn{2}{c}{CIFAR-FS} \tabularnewline
 \cline{2-5}
&     1-shot & 5-shot&   1-shot & 5-shot\\

   \shline
       CC+rot~\cite{gidaris2019boosting}    & 70.53$\pm$0.51 & 84.98$\pm$0.36 & 75.38$\pm$0.31 & 87.25$\pm$0.21\tabularnewline \cline{2-5}
       DC~\cite{yang2021free}$^{*}$    & 74.05$\pm$0.48 & 88.30$\pm$0.32 & --- & ---\tabularnewline \cline{2-5}
     MetaOptNet~\cite{lee2019meta}    & 65.99$\pm$0.72 & 81.56$\pm$0.53& 72.11$\pm$0.96 & 84.32$\pm$0.65\tabularnewline
     \cline{2-5}
     CAN~\cite{hou2019cross}    & 69.89$\pm$0.51 & 84.23$\pm$0.37& --- & ---\tabularnewline
     \cline{2-5}
      FEAT~\cite{ye2020feat}    & 70.80$\pm$0.23 & 84.79$\pm$0.16& --- & ---\tabularnewline \cline{2-5}
      E$^3$BM~\cite{liu2020ensemble}   & 71.20$\pm$0.40 & 85.30$\pm$0.30& --- & ---\tabularnewline \cline{2-5}
      DSN-MR~\cite{simon2020adaptive}    & 67.39$\pm$0.82 & 82.85$\pm$0.56& --- & ---\tabularnewline \cline{2-5}
      FRN~\cite{wertheimer2021few}    & 72.06$\pm$0.25 & 86.89$\pm$0.14& --- & ---\tabularnewline \cline{2-5}
      Tian et.al.~\cite{tian2020rethinking}    & 71.52$\pm$0.69 & 86.03$\pm$0.49& 73.90$\pm$0.80 & 86.90$\pm$0.50\tabularnewline \cline{2-5}
      ConstNet~\cite{xu2021attentional}    & --- & --- & 75.40$\pm$0.20 & 86.80$\pm$0.20\tabularnewline \cline{2-5}
      IEPT~\cite{zhang2020iept}    & 72.24$\pm$0.50 & 86.73$\pm$0.34& --- & ---\tabularnewline \cline{2-5}
      MELR~\cite{fei2020melr}    & 72.14$\pm$0.51 & 87.01$\pm$0.35& --- & ---\tabularnewline \cline{2-5}
      DMF~\cite{xu2021learning}    & 71.89$\pm$0.52 & 85.96$\pm$0.35& --- & ---\tabularnewline \cline{2-5}
      DeepEMD~\cite{zhang2020deepemd}    & 74.29$\pm$0.32 & 87.08$\pm$0.60& --- & ---\tabularnewline \cline{1-5}
     Ours   &  73.48$\pm$0.51 & \bf 87.35$\pm$0.32 & \bf 77.87$\pm$0.49 & \bf 88.94$\pm$0.32\tabularnewline\cline{2-5}
     Ours+DC     &  \bf 75.06$\pm$0.48  & \bf 88.92$\pm$0.32 & --- & ---\tabularnewline
     \hline
  \end{tabular}
 }}
\caption{\label{tab:main_tiered} 5-way few-shot accuracies with $95\%$ confidence interval on \textit{tiered}ImageNet and CIFAR-FS., {*} denotes results reproduced by us.}
\end{table}

\noindent\textbf{Single-domain results.} We adopt three datasets including \textit{mini}ImageNet, \textit{tiered}ImageNet and CIFAR-FS in the single domain setting, where the model is trained on the meta-train set of each dataset and tested on the corresponding meta-test set. The results are shown in Tab.~\ref{tab:main_mini} and Tab.~\ref{tab:main_tiered}. 

In both settings on \textit{mini}ImageNet, our model surpasses all competitors that share the same backbone with us. Particularly, we improve over DeepEMD v2 \cite{zhang2020deepemd} (\emph{i.e.}, the current state-of-the-arts method) by 0.61\% in 1-shot setting. Besides, our methods also performs better (\emph{e.g.}, outperforms FEAT by 4.28\% on 1-shot and 3.03\% on 5-shot), even compared with those with WRN-28 that is much larger and hence has more capacity than Res-12. We observe that this improvement is smaller on 5-shot setting, which may due to extra parameters (\emph{e.g.}, attention module in FEAT and meta-filter in DMF) or specifically-designed but computational expensive algorithm (\emph{e.g.}, the Earth Mover Distance in DeepEMD) that can perform better with more support data. In contrast, our Patchmix, which has the similar procedure of testing with the basic ProtoNet, is more effective and efficient, in terms of prediction power and implementation. This can be contributed to the better representation (more specifically, causally semantic features as will be shown in Sec.~\ref{sec.ablation}) for metric-based classification. Meanwhile, our model also enjoys a suitable confidence interval, which means the robustness against episodes with different categories and difficulty. 

On \textit{tiered}ImageNet and CIFAR-FS the results are nearly consistent with those on \textit{mini}ImageNet (\emph{e.g.}, our model leads by 2.47\% and 1.69\% in 1-shot and 5-shot on CIFAR-FS). Besides, our method is flexible to be further improved when combined with other orthogonal methods, as implied by the improvement when combined with DC (\emph{i.e.}, Ours+DC) which enables the data augmentation on both training and testing phases. Particularly, such a combined method achieves the best performance on \textit{tiered}ImageNet\footnote{Note that as the pre-trained weight for \textit{tiered}ImageNet is not released in DC, we reproduce it by the  (\href{https://github.com/ShuoYang-1998/Few_Shot_Distribution_Calibration}official DC codes. We also adopt the weight provided in S2M2~\cite{mangla2020charting} (\href{https://github.com/nupurkmr9/S2M2_fewshot}{link}) strictly following ~\cite{yang2021free}.}. 

\noindent\textbf{Cross-domain results.} In this setting, we follow the previous methods to train our model on the meta-train set of \textit{mini}ImageNet and test it on the meta-test set of CUB, Cars, Places and Plantae. Since these datasets are fine-grained ones, successful classification mainly requires the model to be able to concentrate on some details which may not be useful in \textit{mini}ImageNet. Therefore the collapsing base IV can have a more serious damage in such a setting. As shown in Tab.~\ref{tab:cross}, our model outperforms the best competitor by at most 5.82\% on 1-shot and 3.14\% on 5-shot among these datasets. This means that when trained on \textit{mini}ImageNet, our model can learn the domain specific information, but also additional knowledge that is useful in other domains.

\noindent\textbf{Unsupervised few-shot learning results.} In this setting, we use \textit{mini}ImageNet as the target dataset. The available information is the same as that in supervised single domain FSL except that labels are not assigned to the base {category} images. We compare our method in Tab.~\ref{tab:unsup} with CACTUs~\cite{hsu2018unsupervised} and UMTRA~\cite{khodadadeh2018unsupervised}. As mentioned in Sec.~\ref{ufsl}, our model is built based on CACTUs-ProtoNet with a changed cluster method, which results in a 1.80\% and 2.56\% improvement on both 1-shot and 5-shot tasks. The superiority further reflects the efficacy of our proposed PatchMix.

\begin{table*}[!t]
 \centering
 {
  \begin{tabular}{ l|cccccccc}
  \hline
 \multirow{2}{*}{Model}    &   \multicolumn{2}{c}{CUB} &   \multicolumn{2}{c}{Cars} &   \multicolumn{2}{c}{Places} &   \multicolumn{2}{c}{Plantae} \tabularnewline
 \cline{2-9}
&   1-shot  &   5-shot &   1-shot  &   5-shot&   1-shot  &   5-shot&   1-shot  &   5-shot \\
   \shline
      RelationNet~\cite{sung2018learning}  & 42.44$\pm$0.77 & 57.77$\pm$0.69 & 29.11$\pm$0.60 & 37.33$\pm$0.68 & 48.64$\pm$0.85 & 63.32$\pm$0.76 & 33.17$\pm$0.64 & 44.00$\pm$0.60 \tabularnewline \cline{2-9}
      GNN~\cite{garcia2017few}  & 45.69$\pm$0.68 & 62.25$\pm$0.65 & 31.79$\pm$0.51 & 44.28$\pm$0.63 & 53.10$\pm$0.80 & 70.84$\pm$0.65 & 35.60$\pm$0.56 & 52.53$\pm$0.59\tabularnewline \cline{2-9}

    LFT~\cite{tseng2020cross}  & 47.47$\pm$0.75 & 66.98$\pm$0.68 & 31.61$\pm$0.53 & 44.90$\pm$0.64 & 55.77$\pm$0.79 & 73.94$\pm$0.67 & 35.95$\pm$0.58 & 53.85$\pm$0.62\tabularnewline \cline{2-9}
    LRP~\cite{sun2021explanation}  & 48.29$\pm$0.51 & 64.44$\pm$0.48 & 32.78$\pm$0.39 & 46.20$\pm$0.46 & 54.83$\pm$0.56 & 74.45$\pm$0.47 & 37.49$\pm$0.43 & 54.46$\pm$0.46\tabularnewline \cline{2-9}

     Ours  & \bf 49.47$\pm$0.45 & \bf 68.90$\pm$0.40 & \bf 33.78$\pm$0.37 &\bf 46.78$\pm$0.43 &\bf 60.65$\pm$0.48 &\bf 77.59$\pm$0.38 &\bf 40.22$\pm$0.39 &\bf 56.01$\pm$0.37

 \tabularnewline
     \hline
  \end{tabular}
 }
 \vspace{0.1in}
\caption{\label{tab:cross}5-way cross-domain few-shot accuracies with $95\%$ confidence interval on CUB, Cars, Places and Plantae.}
\end{table*}

\begin{table}[t]
 \centering
 {
  \begin{tabular}{ l|l|cc}
  \hline
 \multirow{2}{*}{Model}  & \multirow{2}{*}{Clustering} &   \multicolumn{2}{c}{\textit{mini}ImageNet} \tabularnewline
 \cline{3-4}
& &   1-shot  &   5-shot  \\

   \shline
    kNN & \multirow{5}{*}{BiGAN} & 25.56 & 31.10\tabularnewline \cline{3-4}
    linear &  & 27.08 & 33.91\tabularnewline \cline{3-4}
    MLP &  & 22.91 & 29.06\tabularnewline \cline{3-4}
      CACTUs-MAML~\cite{hsu2018unsupervised} &    & 36.24 & 51.28\tabularnewline \cline{3-4}
      CACTUs-ProtoNets~\cite{hsu2018unsupervised} &  & 36.62 & 50.16\tabularnewline \cline{3-4}\hline
    kNN & \multirow{5}{*}{DeepCluster} & 28.90 & 42.25\tabularnewline \cline{3-4}
    linear &  & 29.44 & 39.79\tabularnewline \cline{3-4}
    MLP &  & 29.09 & 39.67\tabularnewline \cline{3-4}
      CACTUs-MAML~\cite{hsu2018unsupervised} &   & 39.90 & 53.97\tabularnewline \cline{3-4}
      CACTUs-ProtoNets~\cite{hsu2018unsupervised} &  & 39.18 & 53.36\tabularnewline \cline{3-4}\hline
      CACTUs-ProtoNets~\cite{hsu2018unsupervised}$^{*}$  & MoCo &  39.18 & 53.36\tabularnewline \cline{3-4}\hline
      UMTRA~\cite{khodadadeh2018unsupervised} & N/A & 39.93 & 50.73\tabularnewline \cline{3-4}\hline
      
     Ours & PatchMoCo  & \bf 41.73 & \bf 55.92\tabularnewline
     
     \hline
  \end{tabular}
 }
 \vspace{0.1in}
\caption{\label{tab:unsup}5-way few-shot accuracies on \textit{mini}ImageNet in unsupervised setting.{*} denotes results produced by us.}

\end{table}

\begin{table*}[htb]
\begin{minipage}{0.24\textwidth} 
\centering 
{
\setlength{\tabcolsep}{1.5mm}{
  \begin{tabular}{ l|cc}
  \hline
Augment & K=1 & K=5 \\
\shline

CutMix & 65.91 & 79.10\\
Mixup & 64.69 & 79.08\\
Man. Mixup & 64.53 & 77.95\\
PatchMix & 68.34 & 83.16\\
\hline
\end{tabular}}}
\end{minipage}
\begin{minipage}{0.24\textwidth} 
 \centering 
{
 \setlength{\tabcolsep}{1.5mm}{
  \begin{tabular}{l|cc}
  \hline
CAN & K=1 & K=5 \\
\shline

base   & 65.05 &  81.42 \\
+CutMix   & 65.27 &  79.53 \\
+PatchMix  & 67.77  &  82.54 \\
+Imp. PM  & 67.79  &  82.76 \\
\hline
\end{tabular}}}
\end{minipage}
\begin{minipage}{0.24\textwidth} 
 \centering 
{
    \setlength{\tabcolsep}{1.5mm}{
  \begin{tabular}{ l|cc}
  \hline
Recons & K=1 & K=5 \\
\shline
w/o &  68.34  &  83.16  \\
vanilla    &  68.59 & 83.25\\
Softmax    &  68.81 & 83.78\\
CGR    &  69.38 & 84.14\\
\hline
\end{tabular}}
}
\end{minipage}
\begin{minipage}{0.24\textwidth} 
 \centering 
{
    \setlength{\tabcolsep}{1.5mm}{
  \begin{tabular}{l|cc}
  \hline
Distill & K=1 & K=5 \\
\shline

vanilla   & 68.61 &  83.46 \\
w/o H  & 68.82  & 83.62 \\
local  & 69.07  & 83.69 \\
global  & 69.38  & 84.14 \\
\hline
\end{tabular}}}
\end{minipage}
\\
\begin{minipage}{0.24\textwidth} 
\centering(a)
\end{minipage}
\begin{minipage}{0.24\textwidth}
\centering(b)
\end{minipage}
\begin{minipage}{0.24\textwidth}
\centering(c)
\end{minipage}
\begin{minipage}{0.24\textwidth}
\centering(d)
\end{minipage}
\vspace{0.1in}
\\
\begin{minipage}{0.24\textwidth} 
\centering 
{
    \setlength{\tabcolsep}{1.5mm}{
  \begin{tabular}{ l|cc}
  \hline
Unsup. & K=1 & K=5 \\
\shline

DC &  40.69 & 54.41  \\
Ours & 41.73 & 55.92\\
\hline
\end{tabular}}}
\end{minipage}
\begin{minipage}{0.24\textwidth} 
\centering 
{
    \setlength{\tabcolsep}{1.5mm}{
  \begin{tabular}{ l|cc}
  \hline
Method & K=1 & K=5 \\
\shline

IFSL & 64.78 & 80.08\\
PatchMix & 68.34 & 83.16 \\
\hline
\end{tabular}}}
\end{minipage}
\begin{minipage}{0.24\textwidth} 
\centering 
{
    \setlength{\tabcolsep}{1.5mm}{
  \begin{tabular}{ l|cc}
  \hline
Strategy & K=1 & K=5 \\
\shline

mix+ori & 68.06 & 82.95\\
all mix & 68.61 & 83.46\\
\hline
\end{tabular}}}
\end{minipage}
\begin{minipage}{0.24\textwidth} 
\centering 
{
    \setlength{\tabcolsep}{1.5mm}{
  \begin{tabular}{ l|cc}
  \hline
Grid size & K=1 & K=5 \\
\shline

$6\times 6$ & 67.95 & 82.69\\
$11\times 11$ & 68.61 & 83.46\\
\hline
\end{tabular}}}
\end{minipage}
\\
\begin{minipage}{0.24\textwidth} 
\centering(e)
\end{minipage}
\begin{minipage}{0.24\textwidth}
\centering(f)
\end{minipage}
\begin{minipage}{0.24\textwidth}
\centering(g)
\end{minipage}
\begin{minipage}{0.24\textwidth}
\centering(h)
\end{minipage}
\\
\vspace{-0.1in}
\caption{\label{tab:ablation} Ablation Studies on \textit{mini}ImageNet 5-way tasks. We show 1-shot(K=1) and 5-shot(K=5) results. (a)
\textbf{PatchMix on our baseline.} We compare PatchMix with other commonly-used data augmentation methods based on our baseline model. (b) \textbf{Plug-in}: We apply our proposed method to CAN~\cite{hou2019cross}. (c) \textbf{CGR}: we compare different instantiations of reconstruction as an auxiliary task for our CGR module. (d) \textbf{Hardness}: we try different implementations of the second stage training, including vanilla distillation, using PatchMix and PatchMix with two kinds of hardness. (e) \textbf{Unsupervise}: we test the proposed substitution for DeepCluster in unsupervised FSL. (f) \textbf{Comparison with IFSL}: we compare our model with IFSL, which also explores causal inference in FSL. (g) \textbf{Mixture strategy}: we try different strategies of using mixed images. (h) \textbf{Pooling}: we compare model trained with and without the last pooling layer in the Res-12 backbone}
 \vspace{-0.2in}
\end{table*}

\begin{table}
 \centering
 {\setlength{\tabcolsep}{2.0mm}{
  \begin{tabular}{ l|cccc}
  \hline
 \multirow{2}{*}{Model} &   \multicolumn{2}{c}{\textit{tiered}ImageNet}&   \multicolumn{2}{c}{CIFAR-FS} \tabularnewline
 \cline{2-5}
&     1-shot & 5-shot&   1-shot & 5-shot\\

   \shline
       w/o    & 72.28 & 86.24 & 76.57 & 88.15\tabularnewline \cline{2-5}
       vanilla    & 72.13 & 86.71 & 76.42 & 88.21\tabularnewline \cline{2-5}
       Softmax    & 72.54 & 86.78 & 77.06 & 87.27\tabularnewline \cline{2-5}
       CGR    & 73.48 & 87.35 & 77.87 & 88.94\tabularnewline 
       
     \shline
     
            vanilla    & 72.56 & 86.39 & 76.97 & 87.83\tabularnewline \cline{2-5}
       w/o H    & 72.78 & 86.67 & 77.02 & 88.10\tabularnewline \cline{2-5}
       local    & 72.81 & 86.81 & 77.30 & 88.52\tabularnewline \cline{2-5}
       global    & 73.48 & 87.35 & 77.87 & 88.94\tabularnewline 
       \hline
  \end{tabular}
 }}
 \vspace{0.1in}
\caption{\label{tab:ablation_tiered} \color{black}{Ablation study of CGR and hardness-aware PatchMix on \textit{tiered}ImageNet and CIFAR-FS.}}
\end{table}

\subsection{Ablation Study}
\label{sec.ablation}

To comprehensively validate the effectiveness of our method, we conduct a series of ablation studies on the design of each sub-module. The accuracies in both 1-shot and 5-shot settings on \textit{mini}ImageNet are reported in Tab.~\ref{tab:ablation}.

\subsubsection{Substitution experiments for PatchMix}
\noindent\textbf{Comparison of PatchMix with other augmentation methods.} Firstly and the most importantly, we testify whether the data augmentation approach in the proposed PatchMix can benefit the few-shot learning. To this end, we compare PatchMix the baseline model that introduced in Sec.~\ref{sec:baseline}, together with three commonly used data augmentation techniques including Mixup, Manifold Mixup and CutMix. For simplicity we omit the performance of baseline and only report the improvement or degeneration of each model compared with baseline. The results in Tab.~\ref{tab:ablation}(a) indicates that the CutMix can only improve the baseline by 0.49\% on 1-shot tasks; while in other settings these methods have no improvements (especially, for manifold Mixup the 5-shot accuracy decreases by 1.39\%). In contrast, our proposed PatchMix can respectively increase by 2.92\% and 3.82\% on 1-shot and 5-shot tasks. Such a noticeable improvement, which can be contributed to the ability of disentangling causal features from others, betokens PatchMix as an effective data augmentation method. 

Moreover, our PatchMix can enjoy a better "neural collapse" property in FSL. Specifically, this property, as observed in \cite{galanti2021role} on FSL,  means that the intra-variance defined as the variance of features from each novel category, can collapse to 0 for a properly trained neural network on sufficient base data. Besides, such novel features form a simplex equiangular tight frame. However, as shown in the experiments in \cite{galanti2021role}, the intra-variance of novel categories can be generally larger than that of the base categories on miniImageNet, which limits the transferability of vanilla training strategy from base categories to the novel ones. To show that our proposed PatchMix can help the FSL models in terms of better neural collapse on novel categories, we visualize the intra-variance of different training methods, including the baseline model where no data augmentation is used as in ~\cite{galanti2021role} and commonly-used augmentation methods like CutMix and Mixup. The results are presented in Fig.~\ref{fig:train_var}(a) and (b).

We observe that the model trained with PatchMix continuously enjoys a larger intra-variance than the baseline model on base categories. The gap is about 17\% in the first 40 epochs and 5\% in the last 10 epochs. As for the novel data, we imitate the testing process to randomly sample 5 novel classes each time for evaluation and calculate the intra-variance with all samples of these classes. Then we repeat this procedure for 20 times and visualize the averaged intra-variance. The result is visualized in Fig.~\ref{fig:train_var}(c). While the intra-variance with PatchMix is larger than that of baseline in the first 7 epochs, it decreases much faster and reaches much smaller value in the end of training. This result, together with the comparison between baseline and PatchMix on CIFAR-FS in Fig.~\ref{fig:var_cifar}, empirically elucidates that our PatchMix can indeed improve the behavior with regard to neural collapse in FSL in terms of generating more collapsed novel features, due to the ability of ours in learning causal features. Specifically, as the novel features are not influenced by the non-causal features which are less correlated to the novel categories, these novel features have decreased variance that is beneficial to separate different categories for classification. 
Beyond the quantitative results, we further visualize learned features of several images from different novel categories that are extracted by models trained with and without the proposed PatchMix, as shown in Fig.~\ref{fig:causal_visual}. We can find that while the model without PatchMix can hardly focus on the target objects, our proposed method can fix this problem, leading to better feature maps.

\begin{figure*}[t]
    \centering
    \includegraphics[scale=0.45]{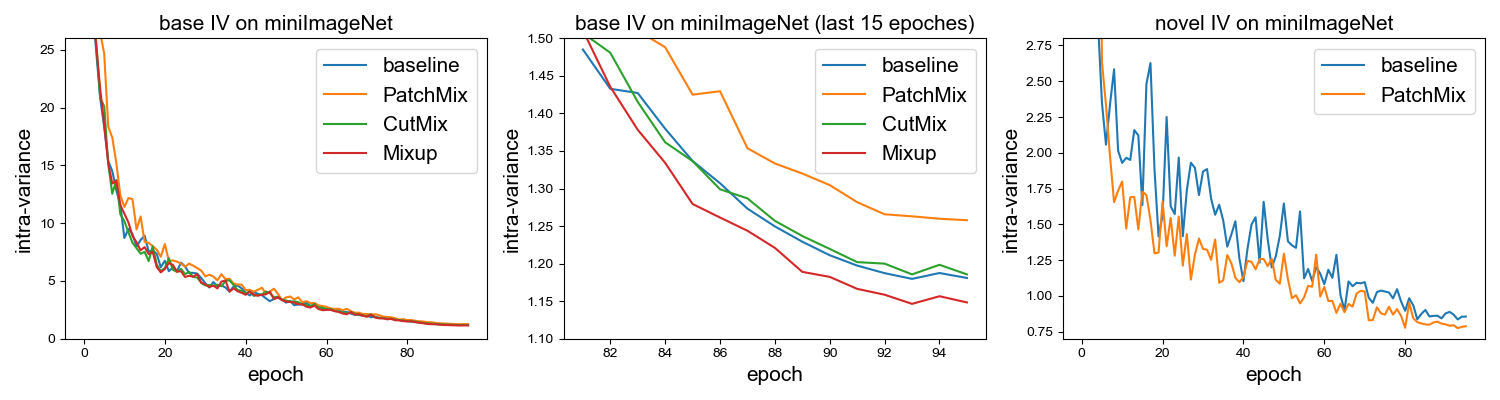}
    \caption{(a) Intra-variance of all 64 base classes on miniImageNet along all epochs when training with 4 different methods. (b) A zoomed version of 80 to 95 epochs of the left image. (c) Comparison of intra-variance of the selected 5 novel classes between baseline and PatchMix. Our PatchMix can not only control a larger IV during training, but also produce better novel class representations with higher accuracy.}
    \label{fig:train_var}
\end{figure*}

\begin{figure}
    \centering
    \includegraphics[scale=0.4]{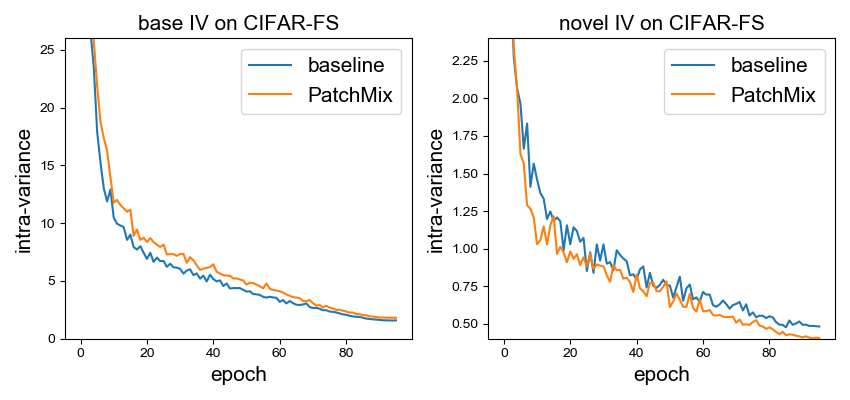}
    \caption{\color{black}{Intra-variance of \textbf{left}: all 64 base classes \textbf{right}: repeatedly sampled 5 novel classes on CIFAR-FS along all epochs when training with and without PatchMix. Similar to the results on \textit{mini}ImageNet, PatchMix leads to both higher base IV and lower novel IV.}}
    \label{fig:var_cifar}
\end{figure}

\begin{figure}
    \centering
    \includegraphics[scale=0.6]{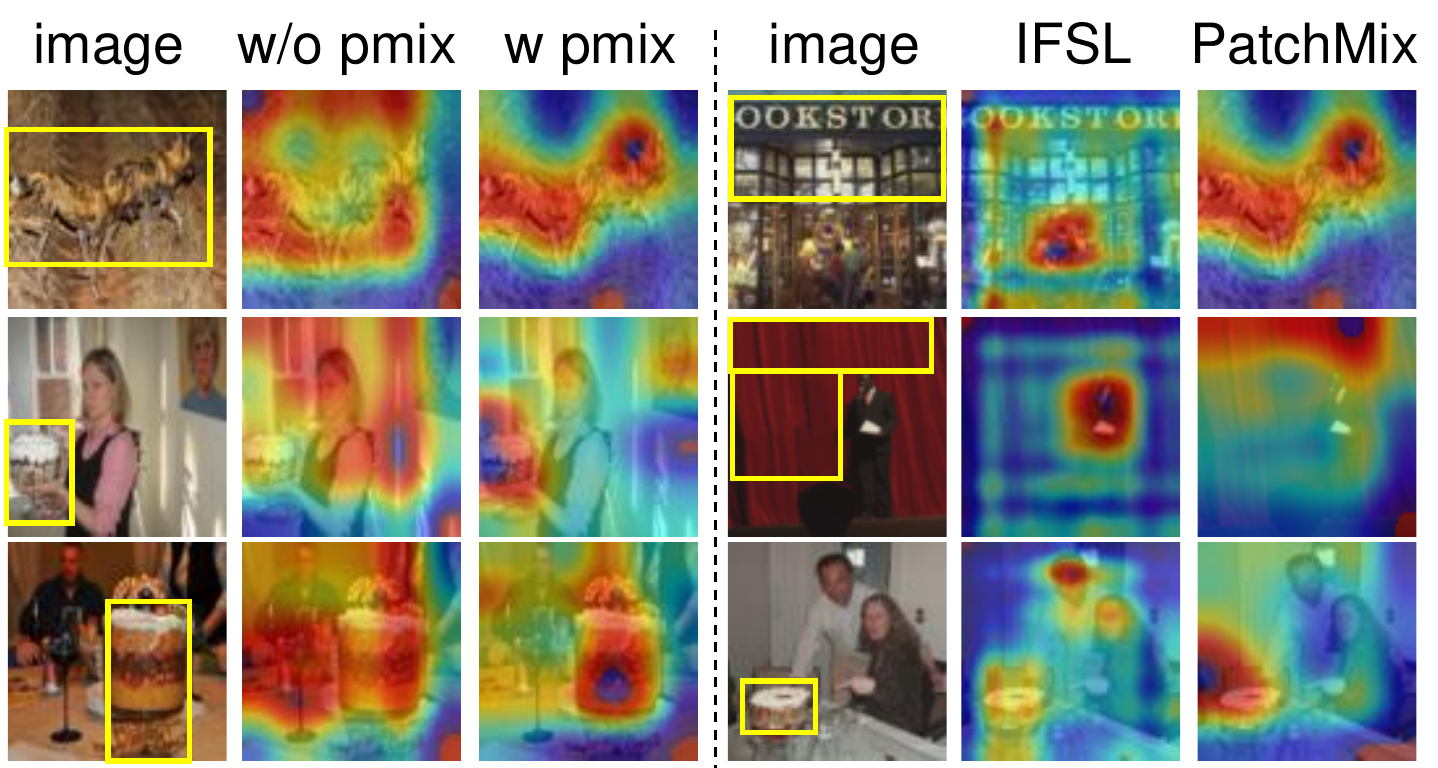}
    \caption{Visualization of features from model trained \textbf{Left}: with or without PatchMix and \textbf{Right}: with IFSL or PatchMix, using images in the novel set of miniImageNet. Objects of interest are highlighted with yellow boxes. {The visualization further illustrates that our PatchMix can disentangle the features to some extent.}}
    \label{fig:causal_visual}
\end{figure}

\noindent\textbf{Applying PatchMix to existing FSL models.} In fact, our proposed PatchMix can be directly applied to any few-shot learning method whose output is composed of a confidence map. To demonstrate the utility of such an application, we employ CAN~\cite{hou2019cross} as a base model and compare the two stages of PatchMix along with CAN against CutMix. Results in Tab.~\ref{tab:ablation}(b) reveal that by applying CutMix to CAN, the 1-shot accuracy is raised but the 5-shot performance is not improved which is consistent with the results on our baseline. In the opposite, utilizing PatchMix and improved PatchMix can boost the accuracy. Concretely, the first stage is better than the basic CAN by 2.72\% and 1.12\% on 1-shot and 5-shot tasks, and the second stage results in a further improvement on both 1-shot and 5-shot tasks. Such results reflect the potential of our method as a plug-in method when solving few-shot learning problems.

\begin{figure}
    \centering
    \includegraphics[scale=0.6]{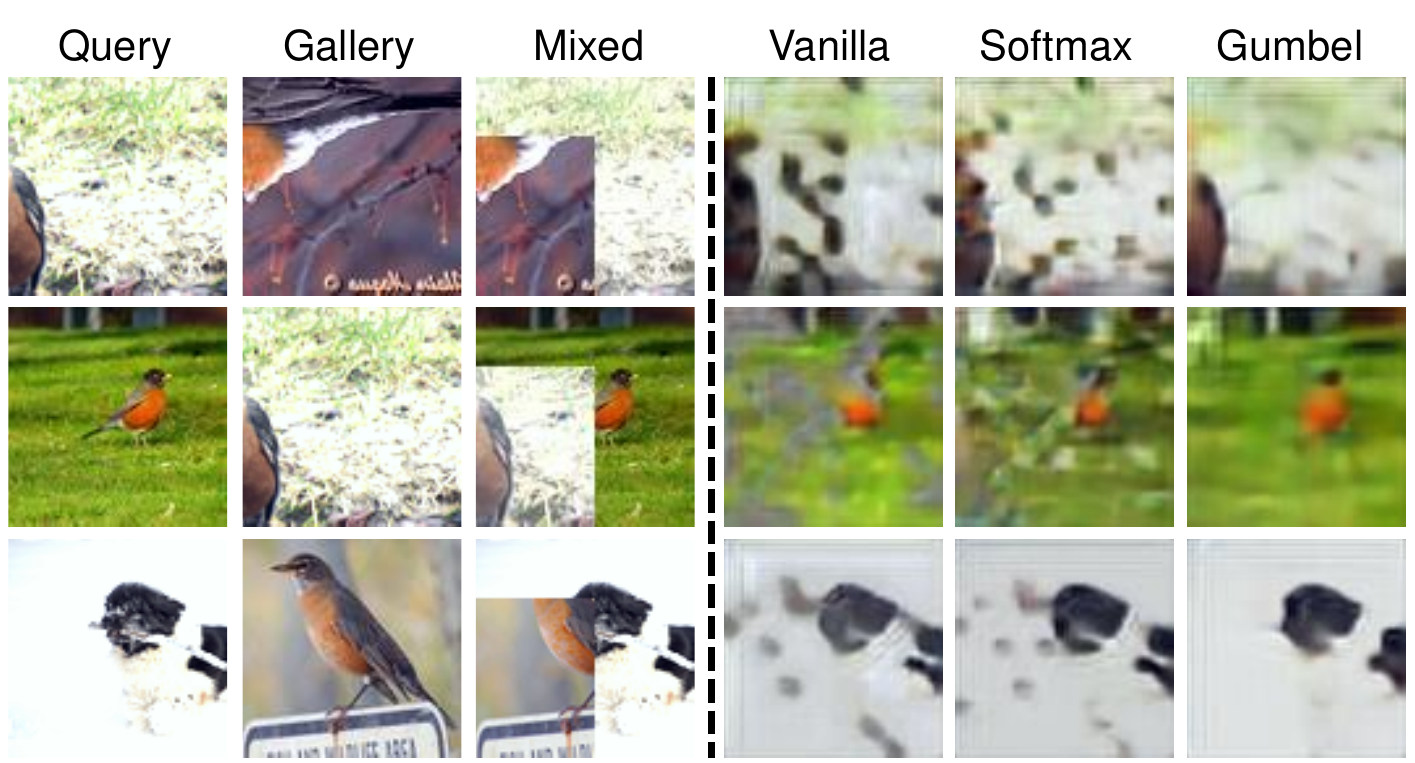}
    \caption{The results of using different methods in our correlation-guided R=reconstruction. We can find that while vanilla reconstruction and the model using softmax produce a lot of artifacts, the images generated by the model using gumbel-softmax as normalization function are more smooth, almost restituting all basic information in the original images.}
    \label{fig:recons}
\end{figure}

\noindent\textbf{Effectiveness of PatchMix in unsupervised representation learning.} As introduced in Sec.~\ref{ufsl}, one of our main adaptation of PatchMix to unsupervised FSL is the substitution of PatchMix to DeepCluster. To show the effectiveness, we train our model with DeepCluster feature instead, whose results are shown in Tab.~\ref{tab:ablation}(e). We find that using DeepCluster the as clustering method with our model is still better than the previous methods, and replacing DeepCluster with PatchMix further improves by 1.04\% and 1.51\% on 1-shot and 5-shot tasks, which reflects the efficacy of two modules in unsupervised setting.

\noindent\textbf{Comparison with IFSL.} In \cite{yue2020interventional} the authors proposed an intervened predictor in the testing pipeline from the perspective of causal inference. To compare the effectiveness, we re-implement this method with Res12 as the backbone and report the comparison result in  Tab.~\ref{tab:ablation}(f). We follow the original paper to perform IFSL based on MTL~\cite{sun2019meta}. As shown, our method can significantly outperform the IFSL on both settings (specifically, respectively improve by 3.10\% and 2.86\% on 1-shot and 5-shot tasks). Besides, this phenomena also holds even without CGR and hardness-aware modules. These results imply the benefit of removing spurious correlation, when the pre-training stage is missing. 

\subsubsection{Ablation study for model variants among different design choices}
\noindent\textbf{Variants for hardness-aware PatchMix.} To illustrate the role of our proposed hardness-aware PatchMix, we compare several variants including vanilla knowledge distillation, using PatchMix without hardness, using local hardness and using global hardness. The results are shown in Tab.~\ref{tab:ablation}(d) and Tab.~\ref{tab:ablation_tiered}. We can find that while using vanilla knowledge distillation can help boost the performance, the improvement {is} relatively small. Moreover, PatchMix without hardness cannot bring {enhancement} on 5-shot tasks. This may {attribute to} that the knowledge imposed by PatchMix has been learned in the teacher model, which is transferred to the student model via distillation. However with more hard examples in the training stage by mixing images from similar classes, we can further promote the distillation process. We suspect that the gap is attributed to the learned discriminative features for classification. 

\noindent\textbf{Variants for correlation-guided reconstruction.} We test the model with and without the proposed CGR. Specifically, the variants include model without reconstruction, model with vanilla reconstruction that query feature map and gallery feature map are directly concatenated and used as the input of the decoder, model with correlation-guided reconstruction where softmax is used as normalization function and our final CGR model where gumbel-softmax is used. As can be seen in Tab.~\ref{tab:ablation}(c) and Tab.~\ref{tab:ablation_tiered}, (1) utilizing vanilla model cannot improve the base model. This means that a simple recovery without explicit modelling of the patch selection cannot help the model learn better representations, which is consistent with the discussion in Sec.~\ref{sec:cgr}. (2) Using softmax as normalization function can bring smaller improvement and sometimes even degeneration. In fact, the reconstruction results of this method is obviously worse than the other two variants, as shown in Fig.~\ref{fig:recons}. We note that this method is the only one that cannot restore the basic shape and colors from the input knowledge. The possible reason is that softmax function makes weight $\alpha_{i,j}\in(0,1)$, which means that it will introduce information from both patches in each position even if one of them is uncorrelated to the target image, which may confuse the decoder. Thus even if the model can learn correct similarity between patches from different images, a bad supervision on the reconstruction hinders the training. (3) Our final choice of gumbel-softmax benefits the model with 0.53\% and 0.52\% higher accuracies on 1-shot and 5-shot tasks on \textit{mini}ImageNet, {\color{black}{and consistent improvement on CIFAR-FS and \textit{tiered}ImageNet}}. This method can provide both better restoration and more precise classification compared to the former two variants, which justifies the efficacy of CGR and the necessity of using gumbel-softmax. The discrete distribution generated by gumbel-softmax does not change the value of original feature, but reorganizing them between two feature maps {instead}. Consequently, a good reconstruction is solely conditioned on a correct selection of patches. 

\noindent\textbf{Mixture strategy.} In Tab.~\ref{tab:ablation}(g) we compare two different strategies of using PatchMix, \emph{i.e.}, using only mixed images or {using} them together with original images. The results demonstrates that abandoning the images before mixture is better by 0.55\% and 0.51\% on 1-shot and 5-shot tasks, and mixing two types of images leads to similar performance of that of the baseline model. One reason may be that such a mixing fails to disentangle causal and non-causal features, in the same way as CutMix as analyzed in Sec.~\ref{sec:anaylysis}.

\noindent\textbf{Does size of feature map affect the performance?} One would ask if it is necessary to modify the {ResNet-12}. We thus compare the model with and without the last max pooling layer in Tab.~\ref{tab:ablation}(h). The results show that deleting the pooling can bring an improvement of 0.74\% on 1-shot and 0.77\% on 5-shot. We argue that the reason is two fold. First, without the max pooling, we can avoid the information loss to some extent, thus comparing the support and query feature in a more detailed way. Second, while larger feature maps may lead to some confusing supervision that some patches do not contain the target object but is guided by the corresponding labels, our PatchMix can alleviate this problem by imposing random information from other categories into these patches, thus making the supervision more robust.

\section{Conclusion}
In this paper we analyze the necessity of learning disentangled causal features, in order to remove the sample selection bias that is commonly met in FSL. To solve these problems, we propose PatchMix to switch the patches and corresponding supervision, which has been theoretically shown to learn causal features by removing the spurious dependency between causal and non causal features across patches. Additionally, we propose two extra modules to enhance PatchMix with more discriminative features. Besides, we present an adaptation of our model for unsupervised FSL. Experiment results among three different settings reveal the efficacy of our proposed method.

{\small
\bibliographystyle{IEEEtran}
\bibliography{egbib}
}

\begin{IEEEbiography}[{\includegraphics[width=1in,
height=1.25in,clip, keepaspectratio]{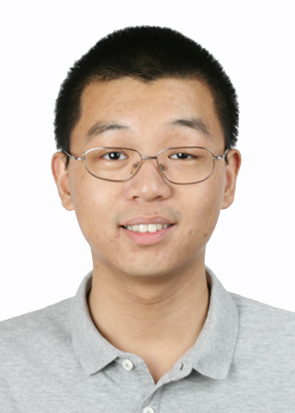}}]{Chengming Xu} 
received his Bachelor's degree in computer science at Fudan University in 2018 and is now a fourth year PhD student majoring in statistics advised by Prof. Yanwei Fu. His research interests include action analysis and few-shot learning.
\end{IEEEbiography}

\begin{IEEEbiography}[{\includegraphics[width=1in,height=1.25in,clip,keepaspectratio]{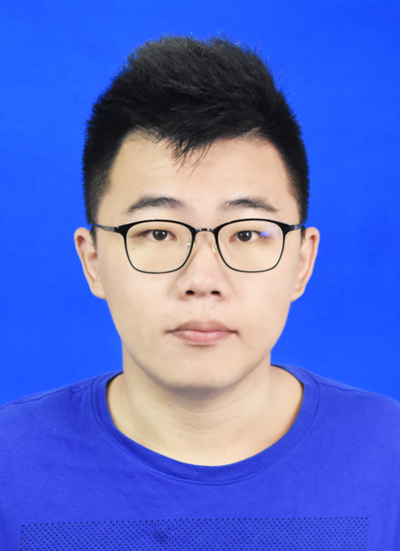}}]{Chen Liu}
is a PhD student at the Department of Mathematics at the Hong Kong University of Science and Technology under the supervision of Prof. Yuan Yao. 
He received the Bachelor degree of Engineering from the School of Mechanical Engineering, Shanghai Jiaotong University, in 2018 and the Master degree of Statistics from the School of Data Science, Fudan University, in 2021.
His current research interests include machine learning and its application to computer vision.
\end{IEEEbiography}

\begin{IEEEbiography}[{\includegraphics[width=1in,height=1.25in,clip,keepaspectratio]{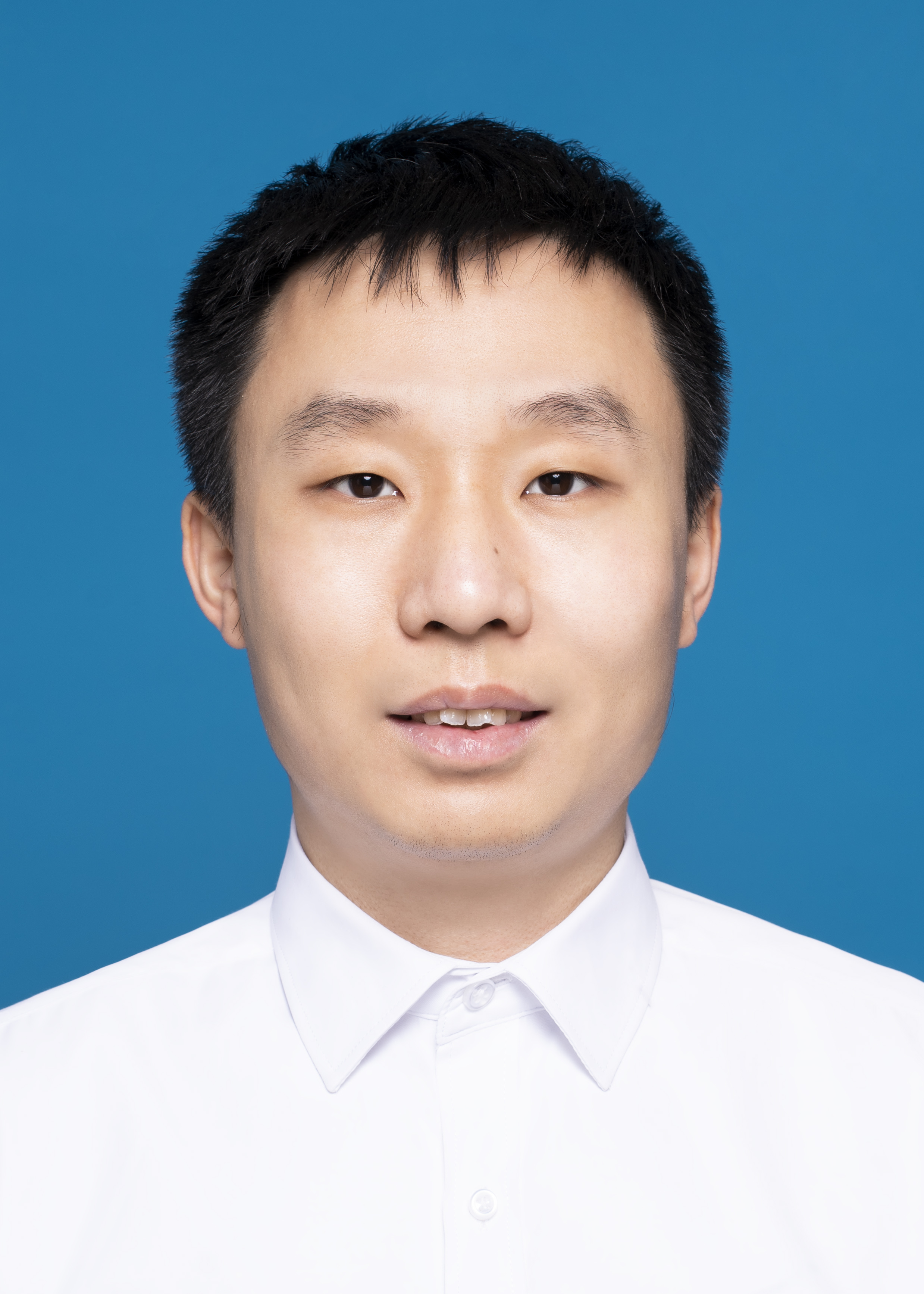}}]{Xinwei Sun}
is currently an assistant professor with School of Data Science, Fudan University. He received his Ph.D in school of mathematical sciences, Peking University in 2018. His research interests mainly focus on statistical machine learning, causal inference, with their applications on medical imaging, computer vision and few-shot learning. 
\end{IEEEbiography}

\begin{IEEEbiography}[{\includegraphics[width=1in,height=1.25in,clip,keepaspectratio]{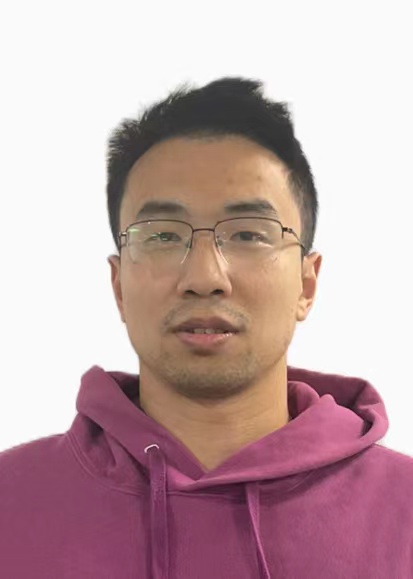}}]{Siqian Yang} recevied his Ph.D. on Computer Science and Technology from Tongji University in 2018. He is currently a researcher at Tencent YouTu Lab, China.
His research interests include image processing, computer vision, and vehicular networks. 
\end{IEEEbiography}

\begin{IEEEbiography}[{\includegraphics[width=1in,height=1.25in,clip,keepaspectratio]{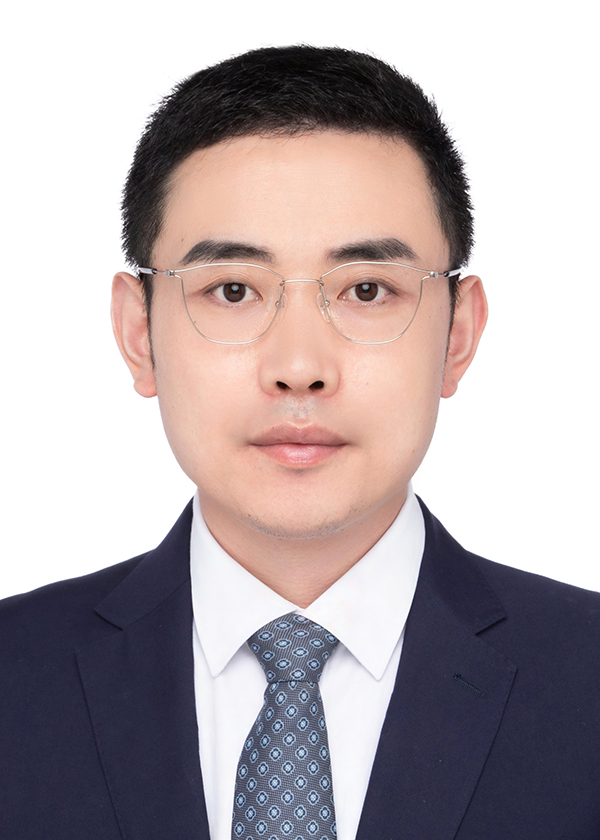}}]{Yabiao Wang} is a Senior Researcher at Tencent Youtu
lab,China. He received his master degree from
Zhejiang University in 2016. He published more
than 30 conference papers including CVPR,
ICCV,ECCV, and AAAI. He won more than 20
challenge titles. His research interests are object
detection/segmentation, few-shot learning and domain
adaptation/generalization. 
\end{IEEEbiography}

\begin{IEEEbiography}[{\includegraphics[width=1in,height=1.25in,clip,keepaspectratio]{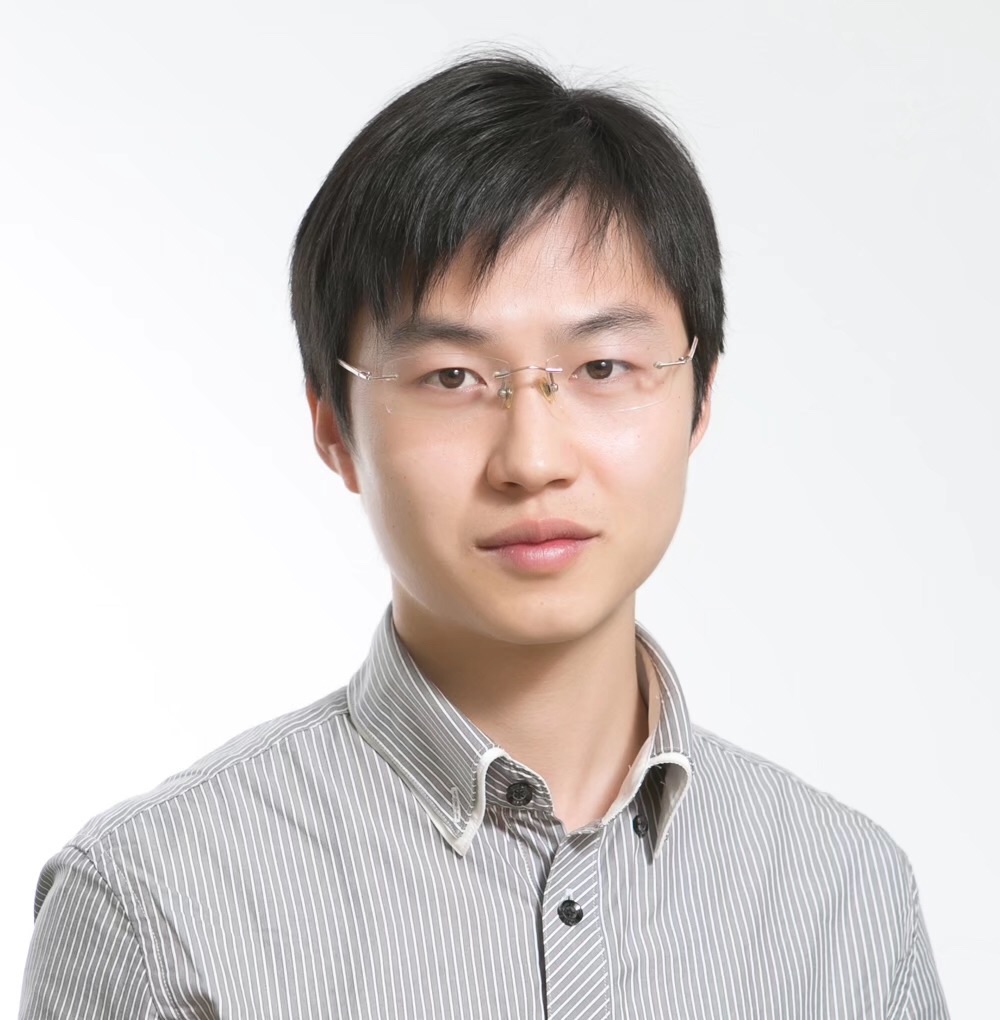}}]{Chengjie Wang} received the B.S degree in computer science from Shanghai Jiao Tong University, China, in 2011, and double M.S. degrees in computer science from Shanghai Jiao Tong University, China and Waseda Univesity, Japan, in 2014. He is currently the Research Director of Tencent YouTu Lab. His research interests include computer vison and machine learning. He has published more than 70 refereed papers on major Computer Vision and Artificial Intelligence Conference and holds over 120 patents in these areas.
\end{IEEEbiography}

\begin{IEEEbiography}[{\includegraphics[width=1in,height=1.25in,clip,keepaspectratio]{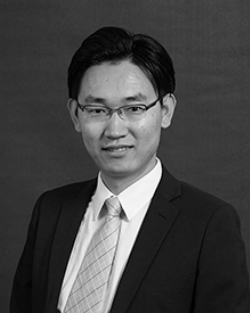}}]{Yanwei Fu} received hist PhD degree from the Queen Mary University of London, in 2014. He worked as  post-doctoral research at Disney Research, Pittsburgh, PA, from 2015 to 2016. He is currently a professor with Fudan University.  
He was appointed as the Professor of Special Appointment (Eastern Scholar) at Shanghai Institutions of Higher Learning in 2017, and   awarded the 1000 Young talent scholar in 2018.
He published more than 100 journal/conference papers including IEEE TPAMI, TMM, ECCV, and CVPR. His research interests are one-shot learning for images and videos, learning based 3D reconstruction for modelling objects/bodies, robotic grasping, and image generation/inpainting/editing. 
\end{IEEEbiography}
\end{document}